\documentclass[letterpaper]{article} 
\usepackage{aaai24}  
\usepackage{times}  
\usepackage{helvet}  
\usepackage{courier}  
\usepackage[hyphens]{url}  
\usepackage{graphicx} 
\usepackage{natbib}  
\usepackage{caption} 
\usepackage[ruled,vlined,linesnumbered]{algorithm2e} 

\SetCommentSty{mycommfont}
\usepackage{amssymb,amsthm}
\usepackage{diagbox}

\newtheorem{theorem}{Theorem}

\newtheorem{lemma}{Lemma}

\newtheorem{definition}{Definition}

\newtheorem{remark}{Remark}
\usepackage{multirow}

\usepackage[utf8]{inputenc}
\usepackage[T1]{fontenc}    
\usepackage{url}            
\usepackage{booktabs}       
\usepackage{amsfonts}       
\usepackage{nicefrac}      
\usepackage{microtype}      

\usepackage{times}
\usepackage{soul}
\usepackage{fdsymbol}
\usepackage{graphicx}
\usepackage{amsmath}

\usepackage{import}
\usepackage{subfigure}

\labelformat{equation}{(#1)}

\usepackage{multirow}
\usepackage{amsmath}
\usepackage[dvipsnames]{xcolor}

\usepackage{newfloat}
\usepackage{listings}
\DeclareCaptionStyle{ruled}{labelfont=normalfont,labelsep=colon,strut=off} 
\title{Causal Discovery from Poisson Branching Structural Causal Model Using High-Order Cumulant with Path Analysis}
\author {
    Jie Qiao\textsuperscript{\rm 1}\equalcontrib,
    Yu Xiang\textsuperscript{\rm 1}\equalcontrib,
    Zhengming Chen\textsuperscript{\rm 1},
    Ruichu Cai\textsuperscript{\rm 1,2}\thanks{Corresponding author.},
    Zhifeng Hao\textsuperscript{\rm 3}
}
\affiliations {
        \textsuperscript{\rm 1}School of Computer Science, Guangdong University of Technology, Guangzhou, China\\
    \textsuperscript{\rm 2}Peng Cheng Laboratory, Shenzhen, China\\
    \textsuperscript{\rm 3}College of Science, Shantou University, Shantou, China\\
    \{qiaojie.chn, thexiang2000, chenzhengming1103,
    cairuichu\}@gmail.com, haozhifeng@stu.edu.cn
}

\begin{document}
\maketitle
\begin{abstract}
Count data naturally arise in many fields, such as finance, neuroscience, and epidemiology, and discovering causal structure among count data is a crucial task in various scientific and industrial scenarios. One of the most common characteristics of count data is the inherent branching structure described by a binomial thinning operator and an independent Poisson distribution that captures both branching and noise. For instance, in a population count scenario, mortality and immigration contribute to the count, where survival follows a Bernoulli distribution, and immigration follows a Poisson distribution.
However, causal discovery from such data is challenging due to the non-identifiability issue: a single causal pair is Markov equivalent, i.e.,  $X\rightarrow Y$ and $Y\rightarrow X$ are distributed equivalent. Fortunately, in this work, we found that the causal order from $X$ to its child $Y$ is identifiable if $X$ is a root vertex and has at least two directed paths to $Y$, or the ancestor of $X$ with the most directed path to $X$ has a directed path to $Y$ without passing $X$. Specifically, we propose a Poisson Branching Structure Causal Model (PB-SCM) and perform a path analysis on PB-SCM using high-order cumulants. Theoretical results establish the connection between the path and cumulant and demonstrate that the path information can be obtained from the cumulant. With the path information, causal order is identifiable under some graphical conditions. A practical algorithm for learning causal structure under PB-SCM is proposed and the experiments demonstrate and verify the effectiveness of the proposed method.

\end{abstract}

\section{Introduction}
Causal discovery from observational data especially for count data is a crucial task that arises in numerous applications in biology \cite{wiuf2006binomial}, economic \cite{weiss2014diagnosing}, network operation maintenance \cite{ijcai2023p0633,cai2022thps}, etc. In online services, for instance, the reason for the number of product purchases is of particular interest, while finding the underlying causal structure among user behavior from purely observational data is appealing and pivotal for online operation. 

Much effort has been made to address the identification of causal structure from observational data \cite{spirtes2000causation,zhang2018learning,glymour2019review,cai2018self}. In particular, constraint-based methods \cite{pearl2009causality,spirtes1995causal}, score-based methods \cite{chickering2002optimal,tsamardinos2006max} identify the causal structure by exploring the conditional independence relation among variables, but these methods only focus on the category domain and can only identify up to the Markov equivalent class  \cite{pearl2009causality}. 
Thus, proper count data modeling is required to further identify the causal structure beyond the equivalence class. 
Recent work by \cite{park2015learning} introduces a Poisson Bayesian network to model the count data and shows that it is identifiable using the overdispersion properties of Poisson BNs. Subsequently, it has been extended by accommodating a broader spectrum of distributions \cite{park2017learning}. In addition, the modeling of the zero-inflated Poisson data \cite{choi2020bayesian} and the ordinal relation data \cite{ni2022ordinal} and its identifiability of causal structure are investigated. However, the majority of these methods model the count data using Bayesian network ignoring the inherent branching structure among the counting relationship which is frequently encountered \cite{weiss2018introduction}.

\begin{figure}[t]
    \centering
    \subfigure[Branching Structure]{
            \includegraphics[width=0.21\textwidth]{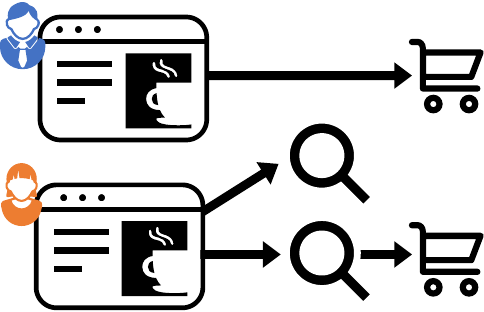}
    }
    \subfigure[Causal Graph]{
            \includegraphics[width=0.145\textwidth]{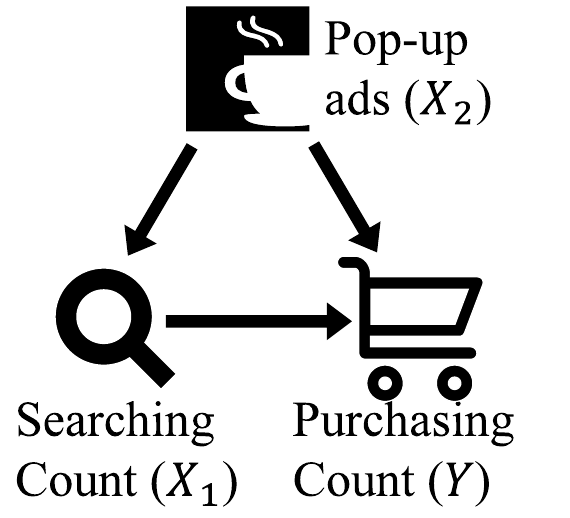}
    }
    \caption{Illustration of branching structure causal modeling.}
    \label{fig:motivate_exp}
\end{figure}

Take Figure \ref{fig:motivate_exp} as an example, the cause of the purchasing event can be inherited from some of the searching events, the pop-up ads event, or exogenously occurs. As a result, the causal relationship among counts constitutes a branching structure that can be modeled by a binomial thinning operator `$\circ$' \cite{steutel1979discrete} with an additive independent Poisson distribution for innovation. That is, the purchasing count ($Y$) is affected by the pop-up ads count ($X_2$) and the searching count ($X_1$) which can be modeled by $Y=a_1\circ X_1 + a_2\circ X_2+\epsilon$ where $a \circ X \coloneq \sum_{n=1}^X\xi^{(a)}_n$, and $\xi^{(a)}_n\sim \text{Bern}(a),\epsilon \sim \text{Pois}$. Generally speaking, the thinning operator models the branching structure that not every click will lead to purchasing while the additional noise models the general count of exogenous events. 
That is, a count represents the random size of an imaginary population, and the thinning operation randomly deletes some of the members of this population while concurrently introducing new immigration. 
This modeling approach finds widespread utility across various domains, notably within the context of the integer-value autoregressive model \cite{weiss2018introduction}, which is first proposed by \citet{al1987first,mckenzie1985some}. Despite its extensive used, how to identify the causal structure in such type of model from purely observational data is still unclear.

To explicitly account for the branching structure, we propose a Poisson Branching Structural Causal Model (PB-SCM). We establish the identifiability theory for the proposed PB-SCM using high-order cumulant with path analysis. Theoretical results suggest that for any adjacent vertex $X$ and $Y$, the causal order is identifiable if $X$ is a root vertex and has at least two directed paths to $Y$, or the ancestor of $X$ with the most directed path to $X$ has a directed path to $Y$ without passing $X$. Based on the results of the causal order we further propose an efficient causal skeleton learning approach featured with FFT acceleration. 
We demonstrate the effectiveness of the proposed causal discovery method using synthetic data and real data.

\section{Poisson Branching Structural Causal Model}
In this section, we first formalize the Poisson branching structural causal model, and then we introduce the preliminary of cumulant and some necessary properties in this model.

\subsection{Problem Formulation}

Our framework is in the causal graphical models. We use $Pa( i) =\left\{j|j\rightarrow i\right\}$, $An( i) =\{j|j\leadsto i \}$ denote the set of parents, ancestors of vertex $i$ in a directed acyclic graph (DAG), respectively, and $An( i,j) =An( i) \cap An( j)$ denote the set of common ancestors of vertex $i$ and vertex $j$. Moreover, we define a \textit{directed path} $P=( i_0 ,i_1 ,...,i_n)$ in $G$ is a sequence of vertices of $G$ where there is a directed edge from $i_j$ to $i_{j+1}$ for any $0\leqslant j\leqslant n-1$ with the coefficient $\alpha _{i_j ,i_{j+1}}$ of each edge. The set of vertices can be arranged in \textit{causal order}, such that no later variable causes any earlier variable.

Now, we show the causal relationship in a causal graph can be formalized as the \textbf{P}oisson \textbf{B}ranching \textbf{S}tructural \textbf{C}ausal \textbf{M}odel (PB-SCM). Let $X =\{X_1 ,\dotsc ,X_{|V|} \}$ denotes a set of random Poisson counts, of which the causal relationship consist of a causal DAG $G(V,E)$ with the vertex set $V=\{1,2,...,|V|\}$ and edge set $E$ such that each causal relation follows the PB-SCM:
\begin{definition}[Poisson Branching Structural Causal Model]
For each random variable $X_i \in X$, let $\epsilon _i \sim \text{Pois} (\mu _i )$ be the noise component of $X_i$, then $X_i$ is generated by: 
\begin{equation}
X_i =\sum _{j\in Pa(i)} \alpha _{j,i} \circ X_j +\epsilon _i ,
\end{equation}
where $\alpha _{j,i} \in (0,1]$ is the coefficient from vertex $j$ to $i$, $Pa( i)$ is the parent set of $X_i$ in $G$, and $\alpha \circ X_i:=\sum\nolimits _{n=1}^{X_i} \xi _n^{(\alpha )}$ is a Binomial thinning operation with $\xi _n^{(\alpha )}\stackrel{\mathrm{i.i.d.}}{\sim }\text{Bern} (\alpha )$, $\text{Bern} (\alpha )$ is the Bernoulli distribution with parameter $\alpha $.
\end{definition}

We further define some graphical concepts.
We use $\mathbf{P}^{i \leadsto j} =\left\{P_k^{i \leadsto j}\right\}_{k=1}^{|\mathbf{P}^{i \leadsto j} |}$ denotes the set of all directed paths from vertex $i$ to $j$, where $P_k^{i \leadsto j}=(i,k_1,k_2,...,k_p,j)$, $p=|P_k^{i \leadsto j}|-2$, denote the $k$-th directed path from vertex $i$ to $j$. For each directed path $\displaystyle P_k^{i \leadsto j}$, we use $\displaystyle A_k^{i \leadsto j} =(\alpha_{i ,k_1}, \alpha_{k_1 ,k_2},\dotsc ,\alpha _{k_p ,j})$ denote the corresponding \textit{coefficients sequence} of path $P_k^{i \leadsto j}$. We let $\mathbf{P}^{i \leadsto i}=\{P^{i \leadsto i}\}$ also be a valid directed path for simplicity.
Besides, we use $A_k^{i \leadsto j} \circ X_i \coloneq \alpha _{k_p ,j} \circ \cdots\circ \alpha_{k_1 ,k_2} \circ \alpha _{i ,k_1} \circ X_i$ denote to perform a consecutive thinning operation on $X_i$ based on the path sequence. 

\noindent\textbf{Goal:} Given i.i.d. samples $\mathcal{D}=\{x_1^{(j)},\dots,x_{|V|}^{(j)}\}_{j=1}^m$ from the joint distribution $P(X)$, our goal is to identify the unknown causal structure $G$ from $\mathcal{D}$, assuming the data generative mechanism follows PB-SCM. 

\subsection{Preliminary}
To address the identification of PB-SCM, cumulant are used in our work for building a connection to the path, providing a solution to the identifiability issue. Here, we recall the definition of cumulant and some basis properties.

\begin{definition}[k-th order joint cumulant tensor]
    The k-th order joint cumulant tensor of a random vector $\displaystyle X =[ X_1 ,...,X_n]^T$ is the $\displaystyle k$-way tensor $\displaystyle \mathcal{T}_X^{( k)}$ in $\displaystyle R^{n\times \cdots \times n} \equiv \left( R^n\right)^k$ whose entry in $\displaystyle ( i_1 ,...,i_k)$ is the joint cumulant:
\begin{equation}\label{eq:cumulant_definition}
\small
\begin{aligned}
 & \mathcal{T}{_X^{( k)}}_{i_1 ,...,i_k} =\operatorname{\kappa } (X_{i_1} ,\dotsc ,X_{i_k} ):=\\
 & \sum _{(B_1 ,\dotsc ,B_L )} (-1)^{L-1} (L-1)!\mathbb{E}\biggl[\prod _{j\in B_1} X_j\biggr] \cdots \mathbb{E}\biggl[\prod _{j\in B_L} X_j\biggr] ,
\end{aligned}
\end{equation}
where the sum is taken over all partitions $( B_1 ,\dots,B_L)$ of the multiset $\{i_1 ,...,i_k\}$.
\end{definition}

In this work, we use the following specific cumulant form:
\begin{definition}[2D slice of joint cumulant tensor]
For a random vector $X$ with $k$-th order joint cumulant tensor $\mathcal{T}_X^{( k)}$ where $k\geq2$, denote its 2D matrix slice of $k$-th order joint cumulant tensor as $\mathcal{C}^{( k)}$, where
\begin{equation}\label{eq:2dcumulant}
\mathcal{C}_{i,j}^{(k)} :=\kappa (X_i ,\underbrace{X_j ,\cdots ,X_j}_{k-1\ \mathrm{times}} ).
\end{equation}
\end{definition}

Cumulant has the property of \textit{multilinearity} such that $\kappa(X+Y,Z_1,\dots)=\kappa(X,Z_1,\dots)+\kappa(Y,Z_1,\dots)$. Furthermore, any cumulant involving two (or more) independent random variables equals zero, i.e., $\kappa(\epsilon_i,\epsilon_j,\dots)=0$ if $\epsilon_i$ and $\epsilon_j$ are independent. More importantly, any two variables in cumulant are exchangeable, e.g., $\kappa(X,Y,\dots)=\kappa(Y,X,\dots)$.

\section{Identifiability}

In this section, we deal with the identification problem of causal structure under PB-SCM. Due to our identifiability result benefit from the `reducibility' of cumulant in Poisson distribution, we first characterize such property in Theorem \ref{th: reduce}. After which, an example is provided to reveal the intrinsic relation between the cumulant and the path in a causal graph under PB-SCM. Based on such connection, we complete the identifiability results that are divided into the case when the cause variable is root (Theorem \ref{th:Identifiability}) and the case when the cause variable is not root (Theorem \ref{thm:identification}).

\begin{figure}[t]
    \centering
    \includegraphics[width=0.43\textwidth]{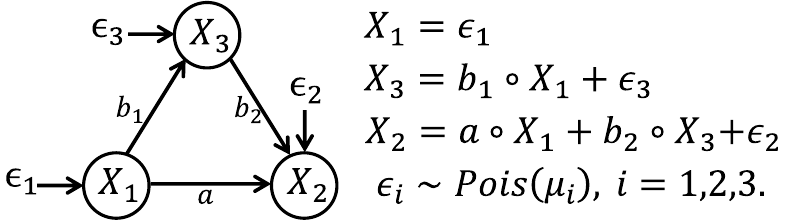}
    \caption{Triangular structure. For simplicity, we denote directed path $P_1 : X_1\xrightarrow{a} X_2$ and $P_2 :X_1\xrightarrow{b_1} X_3\xrightarrow{b_2} X_2$ with sequence of path coefficients $A_1 =(a)$ and $A_2 =(b_1 ,b_2)$.}
    \label{fig:triangular_structure}
\end{figure}

\begin{figure*}[t]
    \centering
    \subfigure[Cumulant decomposition of the causal pair $X_1\leadsto X_2$ where $X_1$ is root.]{
            \includegraphics[width=0.4\textwidth]{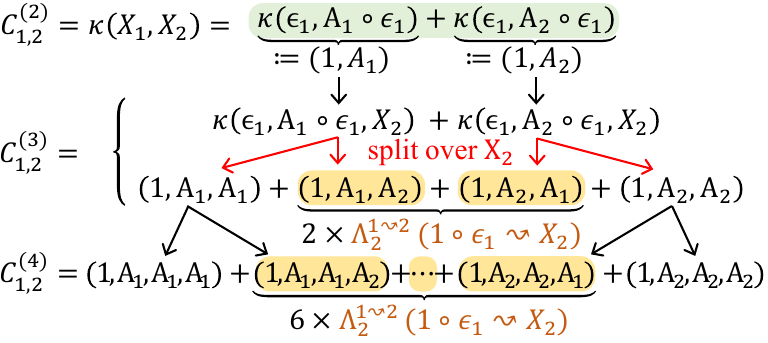}
            \label{pic: X1-X2}
    }
    \subfigure[Cumulant decomposition of the causal pair $X_3\leadsto X_2$ where $X_3$ is not root.]{
            \includegraphics[width=0.5\textwidth]{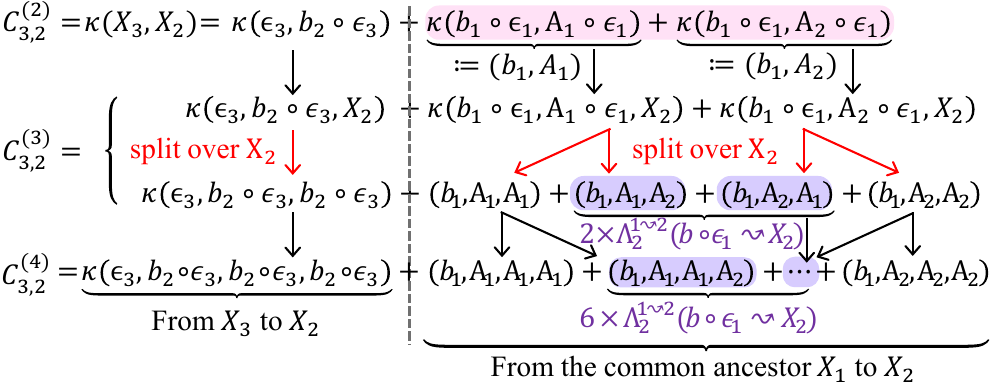}
            \label{pic: X3-X2}
    }
    \caption{Illustration of decomposing the cumulant of causal direction, $\mathcal{C}_{1,2}$ and $\mathcal{C}_{3,2}$, in triangular structure (Fig. \ref{fig:triangular_structure}). For simplicity, we denote $\kappa(\epsilon_i,A_i\circ \epsilon_i,...,A_j\circ \epsilon_i)$ by $(1,A_i,...,A_j)$ and denote $\kappa(b_1\circ\epsilon_i,A_i\circ \epsilon_i,...,A_j\circ \epsilon_i)$ by $ (b_1,A_i,...,A_j).$
    }
    \label{fig: cumulant decomposition}
\end{figure*}

\begin{figure}[t]
    \centering
    \subfigure[Cumulant decomposition of $X_2\!\!\leadsto\!\! X_1$.]{
            \includegraphics[width=0.16\textwidth]{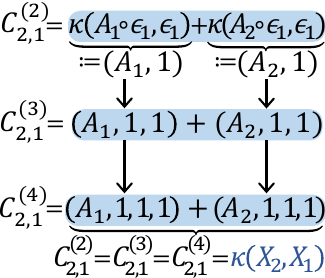}
            \label{pic: X2-X1}
    }
    \subfigure[Cumulant decomposition of $X_2\leadsto X_3$.]{
            \includegraphics[width=0.28\textwidth]{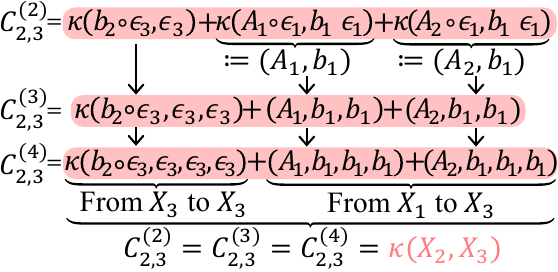}
            \label{pic: X2-X3}
    }
    
    \caption{Illustration of decomposing the cumulant of reverse direction, $\mathcal{C}_{2,1}$ and $\mathcal{C}_{2,3}$, in triangular structure (Fig. \ref{fig:triangular_structure}). }
    \label{fig: cumulant decomposition reverse}
\end{figure}

We first introduce a fundamental property of cumulant in PB-SCM that the cumulant is reducible:
\begin{theorem}[Reducibility]
\label{th: reduce}
Given a Poisson random variable $\epsilon $ and $n$ distinct sequences of coefficients $A_1 ,...,A_n$, we have
\begin{equation}
\begin{aligned}
 & \kappa (\underbrace{A_1 \circ \epsilon ,...,A_1 \circ \epsilon }_{k_1 \ \text{times}} ,...,\underbrace{A_n \circ \epsilon ,...,A_n \circ \epsilon }_{k_n \ \text{times}} )\\
 & =\kappa ( A_1 \circ \epsilon ,\dotsc , A_n \circ \epsilon )
\end{aligned}
\end{equation}
where each $A_i \circ \epsilon $ repeats $k_i\ge 1$ times in the original cumulant and only appears once in the reduced cumulant.
\end{theorem}
Such a result is a generalization of the property of the Poisson distribution since the cumulant of the Poisson distribution is identical in every order.

\subsection{Motivating Example}

Before describing our theoretical results, we use a motivating example to show the challenges of the non-identifiability issues and then introduce the basic intuition regarding in what case and how can we identify the PB-SCM. 

To see the non-identifiability issue, we can show that a reversed model always exists in a two-variable system.
\begin{remark}
    For any two variables causal graph, the causal direction of PB-SCM is not identifiable and a distributed equivalent reversed model exists.
\end{remark}
For instance, consider $X_1\to X_3$ in Fig. \ref{fig:triangular_structure}, the distributed equivalent reverse model satisfies $X_1 =\hat{b}_1 \circ X_3 +\hat{\epsilon }_1$, where $\hat{b}_1 =b_1 \mu _1/(b_1 \mu _1 +\mu _3)$ and $\hat{\epsilon }_1 \sim \text{Pois}( \mu _1 -b_1 \mu _1)$ such that this direction is not identifiable.

Fortunately, we find that the causal direction is still possible to identify in a more general structure. Considering the causal relationship between $X_1$ and $X_2$ in Fig. \ref{fig:triangular_structure}, here we provide an intuitive example to show how to identify such causal direction by utilizing the relationship between cumulant and path.
Considering the cumulant $\mathcal{C}_{1,2}$ with different orders, we can observe different behaviors of cumulant in the causal direction and the reverse direction. Thanks to the reducibility in Theorem \ref{th: reduce}, e.g., $\kappa(A_1\circ \epsilon_1,\epsilon_1)=\kappa(A_1\circ \epsilon_1,\epsilon_1,\epsilon_1)$, the cumulants with different orders for $X_1$ and $X_2$ is shown in Fig. \ref{fig: cumulant decomposition}(a) and Fig. \ref{fig: cumulant decomposition reverse}(a). Interestingly, we have $\mathcal{C}_{2,1}^{(2)}=\mathcal{C}_{2,1}^{(3)}$ in the reverse direction (Fig. \ref{fig: cumulant decomposition reverse}(a)) but $\mathcal{C}_{1,2}^{(2)}\ne \mathcal{C}_{1,2}^{(3)}$ in the causal direction (Fig. \ref{fig: cumulant decomposition}(a)), i.e., there exists an asymmetry in the inequality relations of cumulants. Such asymmetry intriguing possibility to identify the causal order between two variables using the cumulant.

To understand how this asymmetry occurs and hence use it to identify the causal relations. We first discuss the identification in the simple scenario that the cause variable is a root vertex in $G$, and then we generalize such results into the scenario that the cause variable is not root.

\subsection{Identification When Cause Variable Is Root}

We start with the case that the cause variable is root vertex, in which our goal is to identify causal direction even though we do not know it is a root vertex. Recall the previous example, the key of identification is the inequality $\mathcal{C}_{1,2}^{(2)}\ne \mathcal{C}_{1,2}^{(3)}$ rendering an asymmetry for a causal pair. To understand how it occurs, we seek to character and leverage such inequality constraints of cumulants in a causal graph to infer the causal order (Theorem \ref{th:graphical criteria}).

Here, we begin with two basic observations, which illustrate that inequality constraints of cumulants are driven by the number of paths between two variables. As shown in  Fig. \ref{fig: cumulant decomposition}(a), one may see that (i) the decomposition of $\mathcal{C}_{1,2}$ is composed by a series of cumulants of the \textit{common noise} ($\epsilon_1$ in this example) between $X_1$ and $X_2$, which is due to the fact that any cumulant involving two (or more) independent random variables equals zero; (ii) moreover, such decomposition relates to the number of paths between $X_1$ and $X_2$ since $X_2=A_1\circ \epsilon_1+A_2\circ \epsilon_1+b_2\circ \epsilon_3+\epsilon_2$ and by multilinearity, the cumulant will be split exponentially as the order of cumulant increase. With these observations, the reason why $\mathcal{C}_{1,2}^{(2)}\ne \mathcal{C}_{1,2}^{(3)}$ is that there exists more than one path in the causal direction while zero path in the reverse direction, i.e., $|\mathbf{P}^{1\leadsto 2}|=2,|\mathbf{P}^{2\leadsto 1}|=0$. As a result, $\mathcal{C}_{2,1}^{(2)}= \mathcal{C}_{2,1}^{(k)}$ for all $k\geq 2$ order cumulant in the reverse direction.

In the following, we articulate the underlying law of the cumulant in PB-SCM and propose a closed-form solution to it. The first important observation is that due to the reducibility and the exchangeability of cumulant, the $\mathcal{C}_{1,2}^{(k)}$ for $k\geq 3$ is only composed by three distinct cumulants: $\kappa(\epsilon_1, A_1\circ \epsilon_1)$, $\kappa(\epsilon_1, A_2\circ \epsilon_1)$, and $\kappa(\epsilon_1, A_1\circ \epsilon_1, A_2\circ \epsilon_1)$ with varying number of these cumulants.
In particular, if we define the summation of cumulants that only contains one path as $\Lambda_1^{1\leadsto 2}(\epsilon_1\leadsto X_2)\coloneq\kappa(\epsilon_1,A_1\circ \epsilon_1)+\kappa(\epsilon_1,A_2\circ \epsilon_1)$ and the summation of cumulants that contains two paths as $\Lambda_2^{1\leadsto 2}(\epsilon_1\leadsto X_2)\coloneq\kappa(\epsilon_1,A_1\circ \epsilon_1,A_2\circ \epsilon_1)$, we will have the following closed-form solution:
\begin{equation}
\small
    \mathcal{C}_{1,2}^{(4)} =\Lambda _1^{1\leadsto 2} (\epsilon_1 \leadsto X_2 )+\!\!\!\sum\limits _{\overset{m_1 +m_2 =3}{m_1,m_2>0}}\binom{3}{m_1 \ m_2} \Lambda _2^{1\leadsto 2} (\epsilon_1 \leadsto X_2 )
\end{equation}
where $\binom{3}{m_1 \ m_2}$ is the multinomial coefficient, indicating the number of ways of placing $3$ distinct objects into $2$ distinct bins with $m_1$ objects in the first bin, $m_2$ objects in the second bin. As a result, we will eventually have $6 \times \Lambda _2^{1\leadsto 2} (\epsilon_1 \leadsto X_2 )$ as shown in Fig. \ref{fig: cumulant decomposition}(a).
Generally, we define $\Lambda_k^{i\leadsto j}(A\circ \epsilon _i \leadsto X_j )$ as the summation of cumulants that contain $k$ paths from root vertex $i$ to $j$:
\begin{definition}[$k$-path cumulants summation for root vertex]
Given two vertices $i$ and $j$, for $k\leqslant |\mathbf{P}^{i\leadsto j} |$, the $k$-path cumulants summation from vertex $i$ to $j$ is given by:
\begin{equation}\label{eq:lambda_root}
\begin{aligned}
 & \Lambda _k^{i\leadsto j} (A\circ \epsilon _i \leadsto X_j )\\
 & =\!\!\!\!\!\!\sum _{1 \leq l_1 < l_2<...<l_k\leq |\mathbf{P}^{i\leadsto j}|}\!\!\!\!\!\!\!\!
 \kappa (A\circ \epsilon _i ,A_{l_1}^{i\leadsto j} \circ \epsilon _i ,...,A_{l_k}^{i\leadsto j} \circ \epsilon _i ),
\end{aligned}
\end{equation}
where $l_1,\dots,l_k\in \mathbb{Z}^+$, $A$ is an arbitrary sequence of coefficients. For $k >|\mathbf{P}^{i\leadsto j} |$, $\Lambda^{i\leadsto j}_k\equiv 0$ and for $k=1$, $\Lambda^{i\leadsto i}_1 (A\circ \epsilon _i \leadsto X_i )= \kappa(A\circ \epsilon _i,\epsilon _i)$, and $k>1,\Lambda^{i\leadsto i}_k\equiv 0$.
\end{definition}

Intuitively, Eq. \ref{eq:lambda_root} is a summation of all cumulants that contain $k$ paths information from vertex $i$ to $j$ , and $\Lambda^{i\leadsto i}_1$ denotes the relation from the noise to itself. Based on the $k$-path cumulants summation, $\mathcal{C}^{(n)}_{i,j}$ can be decomposed as follows:

\begin{theorem}\label{th:lambda and cumulant root}
For any two vertices $i$ and $j$ where $i$ is root vertex, i.e., vertex $i$ has an empty parent set, the 2D slice of joint cumulant $\mathcal{C}_{i,j}^{(n)}$ satisfies:
\begin{equation}\label{eq:cumulant_lambda_root}
\begin{aligned}
\mathcal{C}_{i,j}^{(n)} & =\sum\nolimits _{k=1}^{n-1}\!\!\!\!\sum\limits _{\overset{m_1 +\cdots +m_k =n-1}{m_l  >0}}\!\!\binom{n-1}{m_1 \ m_2 \cdots m_k} \Lambda_k^{i \leadsto j}( 1\!\circ\! \epsilon_i \!\leadsto\! X_j).
\end{aligned}
\end{equation}
where $\binom{n-1}{m_1 \ m_2 \cdots m_k} =\frac{(n-1)!}{m_1 !m_2 !\cdots m_k !}$ is the multinomial coefficients.
\end{theorem}
Theorem \ref{th:lambda and cumulant root} plays an important role in the identification of the causal order as it introduces the connection between the joint cumulant and path information. Moreover, since every order of the 2D slice joint cumulant can be obtained by Eq. \ref{eq:2dcumulant}, and thus every order of $\Lambda_k$ can also be obtained by solving the equation in Eq. \ref{eq:cumulant_lambda_root}. 
By using $\Lambda_k$ we are able to understand the identifiability in the following theorem:
\begin{theorem}[Identifiability for root vertex]\label{th:Identifiability}
    For any vertex $i$ and $j$, where $i$ is the root vertex in graph $G$, if $\mathcal{C}_{i,j}^{(3)}-\mathcal{C}_{i,j}^{(2)}\ne 0$, then $\mathcal{C}_{j,i}^{(3)}-\mathcal{C}_{j,i}^{(2)}=0$ and $X_i$ is the ancestor of $X_j$.
\end{theorem}
Intuitively, based on Theorem \ref{th:lambda and cumulant root}, we have $\mathcal{C}_{i,j}^{(3)}-\mathcal{C}_{i,j}^{(2)} = \Lambda_2^{i\leadsto j}(1\circ\epsilon_i \!\!\leadsto\!\! X_j)$, and thus $\mathcal{C}_{i,j}^{(3)}-\mathcal{C}_{i,j}^{(2)}\ne 0$ indicates that there exists more than one path from $i$ to $j$ than the reverse direction. That is, the causal direction for root vertex is identifiable if there are at least two directed paths:
\begin{theorem}[Graphical Implication of Identifiability for Root Vertex]\label{th:graphical criteria}
    For a pair of vertices $i$ and $j$ in graph $G$, if vertex $i$ is a root vertex and exists at least two directed paths from $i$ to $j$, i.e., $|\mathbf{P}^{i\leadsto j}|\geq 2$, then the causal order between $i$ and $ j$ is identifiable. 
\end{theorem}

\subsection{Identification When Cause Variable Is Not Root}
In this section, we aim to generalize the identification result from the root vertex to the non-root vertex.

When vertex $i$ is not root, the main difference is that there might exist more than one common noise between two variables due to the possible common ancestor. Therefore, one may extend the result from the root vertex by considering each noise term as the separated root vertex. We present a general version of $k$-path cumulants summation as follows, which can be expressed as the aggregation of the $k$-path cumulants summations for the root vertices.

\begin{definition}[$k$-path cumulants summation]
The $k$-path cumulants summation from vertex $i$ to vertex $j$ is given by:
\begin{equation}
    \begin{aligned}
  \tilde{\Lambda}_k( X_i \!\leadsto\! X_j) =&\Lambda _k^{i\leadsto j}( 1\!\circ\! \epsilon_i \!\leadsto\! X_j)\\
+ & \!\!\!\!\!\sum_{m\in An( i,j)\cup \{j\}}\!\!\sum _{h=1}^{|\mathbf{P}^{m\leadsto i} |} \Lambda _k^{m\leadsto j}\left(A_h^{m\leadsto i} \!\circ\! \epsilon_m \!\leadsto\! X_j\right).
\end{aligned}
\end{equation}
where $\Lambda_k$ is the $k$-path cumulants summation for root vertex, $|\mathbf{P}^{m\leadsto i}|$ is the number of directed paths from $m$ to $i$.
\end{definition}
With the general $k$-path cumulants summation, the general joint cumulant can be decomposed as follows:
\begin{theorem}\label{th:lambda and cumulant}
For any two vertices $i$ and $j$, the 2D slice of joint cumulant $\mathcal{C}_{i,j}^{(n)}$ satisfies:
\begin{equation}\label{eq:cumulant_lambda}
\small
\begin{aligned}
\mathcal{C}_{i,j}^{(n)} & =\sum\nolimits _{k=1}^{n-1}\!\!\!\!\sum\limits _{\overset{m_1 +\cdots +m_k =n-1}{m_l  >0}}\!\!\binom{n-1}{m_1 \ m_2 \cdots m_k} \tilde{\Lambda}_k( X_i \!\leadsto\! X_j),
\end{aligned}
\end{equation}
where $\binom{n-1}{m_1 \ m_2 \cdots m_k} =\frac{(n-1)!}{m_1 !m_2 !\cdots m_k !}$ is the multinomial coefficients.
\end{theorem}

To see the connection with the case of root vertex, we take $X_3\to X_2$ in Fig. \ref{fig:triangular_structure} as example. Since $X_3$ can be expressed as $X_3=b_1\circ \epsilon_1+\epsilon_3$, as shown in Fig. \ref{fig: cumulant decomposition}(b), we can separate the cumulant into two parts $\kappa (\epsilon_3,X_2)$, $\kappa (b_1\circ \epsilon_1,X_2)$, which can be considered as the cumulant starting from vertex $X_3$ to $X_2$ and $X_1$ to $X_2$, respectively. As a result, the general $k$-path cumulants summation can be expressed as the aggregate of all different $\Lambda_k$ starting with the corresponding noise terms. For instance, for $X_3\to X_2$ in Fig. \ref{fig:triangular_structure}, we have:
\begin{equation}\label{eq:example_lambda_3to2}
    \begin{aligned}
 & \tilde{\Lambda }_2( X_3 \leadsto X_2)\\
 & =\underbrace{\Lambda_2^{3\leadsto 2}( 1\circ \epsilon_3 \leadsto X_2)}_{=0} + \underbrace{\Lambda _2^{1\leadsto 2}( b_1 \circ \epsilon_1 \leadsto X_2)}_{=\kappa ( b_1 \circ \epsilon_1 ,A_1 \circ \epsilon_1 ,A_2 \circ \epsilon_1)}  \neq 0,
\end{aligned}
\end{equation}
where Eq. \ref{eq:example_lambda_3to2} contains two different terms starting from $\epsilon_3$ and $\epsilon_1$, respectively. In particular, since there only exists one directed path from $X_3$ to $X_2$, $\Lambda_2^{3\leadsto 2}$ is zero while $X_1$ to $X_2$ has two paths and thus $\Lambda_2^{1\leadsto 2}$ is not zero. Similarly, for the reverse direction, we have
\begin{equation}
    \begin{aligned}
  \tilde{\Lambda }_2( X_2 \!\!\leadsto\!\! X_3)&
  =\underbrace{\Lambda _2^{2\leadsto 3}( 1 \circ \epsilon_2 \!\!\leadsto\!\! X_3)}_{=0}+\underbrace{\Lambda _2^{3\leadsto 3}( b_2 \circ \epsilon_3 \!\!\leadsto\!\! X_3)}_{=0}\\
 & \!\!\!\!\!\!\!\!\!\!\!\!+ \underbrace{\Lambda _2^{1\leadsto 3}( A_1 \circ \epsilon_1 \!\!\leadsto\!\! X_3)}_{=0}+ \underbrace{\Lambda _2^{1\leadsto 3}( A_2 \circ \epsilon_1 \!\!\leadsto\!\! X_3)}_{=0}=0,
\end{aligned}
\end{equation}
where $\tilde{\Lambda}_2$ is zero since there are $0$ directed path from $X_2$ to $X_3$ and only $1$ directed path from $X_1$ or $\epsilon_3$ to $X_3$. 
Intuitively, the general $k$-path cumulants summation $\tilde{\Lambda}(X_i\leadsto X_j)$ captures the number of directed paths from the common ancestor to $j$. Moreover, for any two adjacency vertex $i \to j$ and their common ancestor $m$, the number of directed paths from $m$ to $j$ is greater or equal to that from $m$ to $i$, and thus, the causal order can be identified using the following strategy:

\begin{theorem}[Identification of PB-SCM]\label{thm:identification}
    If there exist $k\in \mathbb{Z}^+$ such that $\tilde{\Lambda}_k(X_i\!\leadsto\! X_j)\ne0$ and $\tilde{\Lambda}_k(X_j\!\leadsto\! X_i)=0$ for any two adjacency vertex $i$ and $j$, then $X_i$ is the parent of $X_j$.
\end{theorem}

In addition, the $k$-path cumulants summation $\tilde{\Lambda}_k(X_i\leadsto X_j)$ will be `dominated' by the variables (might be the common ancestor or $i$ itself) that has the most paths to $j$ since it is the aggregation of all the directed paths from both common ancestor and $i$. Therefore, for a non-root vertex, it is possible to be non-identifiable by Theorem \ref{th:Identifiability} if the dominant variable is the common ancestor. Specifically, we provide the graphical implication of such identifiability given as follows:
\begin{theorem}[Graphical Implication of Identifiability]\label{th:graphical criteria 2}
    For a pair of causal relationship $i \to j$. The causal order of $i,j$ is identifiable by Theorem \ref{thm:identification}, if (i) vertex $i$ is a root vertex and $|\mathbf{P}^{i\leadsto j}|\geq 2$; or (ii) there exists a common ancestor $ {k\in \underset{l}{\arg\max}\left\{|\mathbf{P}^{l\leadsto i} ||l\in An( i,j)\right\}}$ has a directed path from $k$ to $j$ without passing $i$ in $G$. 
\end{theorem}

One of the examples is given in Fig. \ref{fig:illustration identifiability}, in which Fig. \ref{before_add_edge} is not identifiable but Fig. \ref{after_add_edge} is identifiable. The reason is that $Z$ is the dominant common ancestor of $X,Y$, and all directed paths from $Z$ to $Y$ will pass $X$ making it unidentifiable based on Theorem \ref{th:graphical criteria 2}. In contrast, Fig. \ref{after_add_edge} includes an additional directed path $Z\to C \to Y$ without passing $X$ making $X\to Y$ identifiable. This intriguingly implies that a denser structure would facilitate the effectiveness of our method.

\begin{figure}[t]
	\centering
	\subfigure[Not identifiable.]{
	\includegraphics[width=0.16\textwidth]{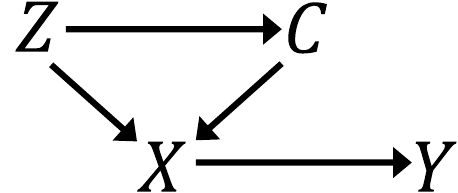}
	\label{before_add_edge}
}~~~~~~~~
	\subfigure[Identifiable.]{
		\includegraphics[width=0.16\textwidth]{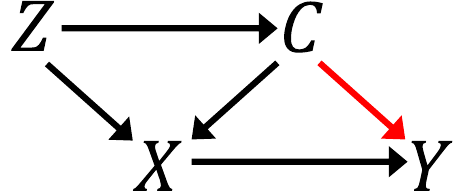}
		\label{after_add_edge}
	}
 \caption{Illustration of the identifiability of $X\to Y$.}
 \label{fig:illustration identifiability}
\end{figure}

Generally speaking, once the causal order is identified, one may identify the complete causal structure by orienting edges based on the causal order in the causal skeleton. Such implementation will be provided in the next section. By this, the identifiability of causal structure under PB-SCM is answered.

\section{Learning Casual Structure For PB-SCM}
In this section, we propose a causal structure learning algorithm for PB-SCM. Our method involves two steps: learning the skeleton of DAG $G$ and inferring the causal direction using the results developed in Theorem \ref{thm:identification}.

\subsubsection{Learning Causal Skeleton}To learn the causal skeleton, instead of using the constraint-based method, we propose a likelihood-based method. This boosts sample efficiency as the likelihood of PB-SCM captures its branching structure but the constraint-based method does not. 

Given a set of count data $\mathcal{D}$ and model parameters $\Theta =\left\{\mathbf{A} =[ \alpha _{i,j}] \in [ 0,1]^{|V|\times |V|} ,\boldsymbol{\mu} =[\mu_i] \in \mathbb{R}_{\geq 0}^{|V|}\right\}$, the log-likelihood is Markov respect to $G$, that is $\mathcal{L} (G,\Theta ;\mathcal{D} )=\sum _{j=1}^{|\mathcal{D}|} \sum _{i=1}^{|V|} \log P_{\Theta}\left( X_i =x_i^{( j)} |X_{Pa( i)} =x_{Pa( i)}^{( j)}\right)$.
However, calculating the likelihood directly using the probability mass function is costly. Therefore, we propose to calculate the probability mass function by using the probability-generating function (PGF). In detail, for each conditional distribution of $X_i$, the likelihood can be calculated as follows:

\begin{theorem}
    Let $G_{X_i |X_{Pa( i)}}\!( s)$ be the PGF of random variable $X_i$ given its parents variable $X_{Pa( i)}$, we have:
    \begin{equation}
    \label{eq: pgf}
    \small
    \begin{aligned}
     & P( X_i =k|X_{Pa( i)} =x_{Pa( i)})=\frac{1}{k!}\frac{\partial ^k G_{X_i |X_{Pa( i)}}( s)}{( \partial s)^k}\Bigl|_{s=0}\\
     & =\!\!\sum _{
        t_i +\!\!\sum\limits_{j\in Pa( i)}\!\! t_j =k
        }\!\! \frac{\mu _i^{t_i}\exp( -\mu _i)}{t_i !}\!\!\prod\limits _{j\in Pa(i)}\! \!\!\frac{( x_j)_{t_j}\alpha _{j,i}^{t_j} (1-\alpha _{j,i} )^{x_j -t_j}}{t_j !},
    \end{aligned}
    \end{equation}
where $t_j \leq x_j$, $\displaystyle ( x_j)_{t_j} \!\coloneq\!\frac{x_j !}{( x_j -t_j) !}$ is the falling factorial, $\mu _i \!=\!E[ \epsilon _i]$, and $\epsilon _i$ is the noise component of $X_i$.
\end{theorem}
The result of Eq. \ref{eq: pgf} can be converted to a polynomial coefficient after taking polynomial multiplication, which can be accelerated via Fast Fourier Transform (FFT) \cite{cormen2022introduction}. A detailed discussion is given in the supplement.

Generally, the likelihood-based method will tend to produce excessive redundant causal edges. Such effect can be alleviated by introducing the Bayesian Information Criterion (BIC) penalty $d\log( m)/2$ into the $\mathcal{L} (G,\Theta;\mathcal{D} )$, where $d$ is the number of edge of $G$ and $m$ is the size of dataset $\mathcal{D}$. The penalized objective function is updated as follows:
\begin{equation}\label{eq:likelihood}
\mathcal{L}_p (G,\Theta ; \mathcal{D} )=\mathcal{L} (G,\Theta ;\mathcal{D} )-{d\log( m)}/{2}
\end{equation}
We maximum the objective function $\mathcal{L}_p (G,\Theta;\mathcal{D} )$ by using a Hill-Climbing-based algorithm as shown in Lines 2-6 of Algorithm 1. It mainly consists of two phases. 
First, we perform a structure searching scheme by taking one step adding, deleting, and reversing the graph $G^{*}$ in the last iteration, i.e., in Line 4, $\mathcal{V}\left( G^{*}\right)$ represents a collection of the one-step modified graph of $G^{*}$. Second, by fixing the graph $G'$, we estimate the parameter $\Theta '$ of the model via optimizer with initial values from approximated covariance estimates and then calculate the $\mathcal{L}'_p (G',\Theta ';\mathcal{D} )$ in Lines 5. Iterating the two steps above until the likelihood no longer increases. In the end, we transform $G^{*}$ into a skeleton (Line 6). The correctness of such a procedure can be guaranteed by the consistent property of BIC which is discussed in \cite{chickering2002optimal}.

\subsubsection{Learning Causal Direction}
Given the learned skeleton, we orient each undirected edge using the $k$-path cumulants summation, according to Theorem \ref{thm:identification}. In detail, for each undirected edge $( i,j) \in E$, we calculate $\tilde{\Lambda} _k( X_i \leadsto X_j)$ and $\tilde{\Lambda}_k( X_j \leadsto X_i)$ for $k=1,\dots,K$ until one of them being zero or $k$ reaches the upper limit $K$. We then orient the direction based on Theorem \ref{thm:identification} (Lines 11-14). 

To assess whether $\tilde{\Lambda }_k$ is equal to 0, a bootstrap hypothesis test is conducted \cite{efron1994introduction} while a threshold can be used for orientation once such testing fails. In detail, we calculate the statistic $\tilde{\Lambda }^{+}_k$ from $N$ resampling dataset $\mathcal{D}^{+}\in \{\mathcal{D}_i^{+}|\mathcal{D}_{i=1,..,N}^{+}\subset\mathcal{D},\}$. Then, we estimate the distribution $P(\tilde{\Lambda }^{+}_k)$ by kernel density estimator and centralize it to mean zero. Finally, the p-value of $\tilde{\Lambda }_k$ from the original dataset can be obtained.

\noindent\textbf{Complexity Analysis} We provide the complexity of calculating likelihood in the worst cases---when graph is complete. Specifically, the complexity of Eq. \ref{eq:likelihood} is $
    \mathcal{O}(\sum _{j=1}^m\sum _{i=1}^{|V|}\frac{\left( |V|+x_i^{( j)} -i\right) !}{( |V|-i) !x_i^{( j)} !})$, by using FFT acceleration, this complexity can be reduced to ${\mathcal{O}(\sum _{j=1}^m\sum _{i=1}^{|V|}( |V|-i+1)^2 x_i^{( j)}\log( |V|-i+1)^2 x_i^{( j)})}$, where $\displaystyle m$ is the sample size.

\begin{algorithm}[t]
    \caption{Causal Discovery for PB-SCM}
    \KwIn{{Data set $\mathcal{D}$}, Max order $K$}
    \KwOut{Learning Causal Graph $G$}
    
     $G' \leftarrow empty\ graph, \mathcal{L}_p^* \leftarrow -\infty;$\\
    \tcp{Learning Causal Skeleton} 
    \While{$\mathcal{L}_p^*(G^*, \Theta ^*; \mathcal{D})<\mathcal{L}_p'(G', \Theta '; \mathcal{D})$}{
        $G^* \leftarrow G'$ with largest $\mathcal{L}_p'(G', \Theta '; \mathcal{D})$\\
        \For{every $G'\in \mathcal{V}(G^*)$}{
            Estimate $\Theta '$ and record score $\mathcal{L}_p'(G', \Theta '; \mathcal{D})$\\
        }        
    }
    $G\gets$ Transfer $G^*$ to a skeleton\\
    \tcp{Learning Causal Direction}
    \For{each pair $X_i - X_j\in G$}{
    \For{$k\leftarrow 1:K$}{
        Obtain $\tilde{\Lambda} _k$ at each side by solving Eq. \ref{eq:cumulant_lambda}\\
        Test whether $\tilde{\Lambda} _k$ equal to $0$ for each side\\
        \uIf{$\tilde{\Lambda} _k(X_i\!\!\leadsto\!\! X_j)\neq0 \land \tilde{\Lambda}_k(X_j\!\!\leadsto\!\! X_i)=0 $}{
            Orient $``X_i\rightarrow X_j"$ in $G$
        }
        \If{$\tilde{\Lambda} _k(X_i\!\!\leadsto\!\! X_j)=0 \land \tilde{\Lambda}_k(X_j\!\!\leadsto\!\! X_i)\neq0 $}{
            Orient $``X_i\leftarrow X_j"$ in $G$
        }
        }
    }
    \textbf{Return} $G$
\end{algorithm}

\section{Experiment}
\begin{figure*}[t]
	\centering
	\subfigure[Sensitivity to Avg. Indegree Rate]{
	\includegraphics[width=0.31\textwidth]{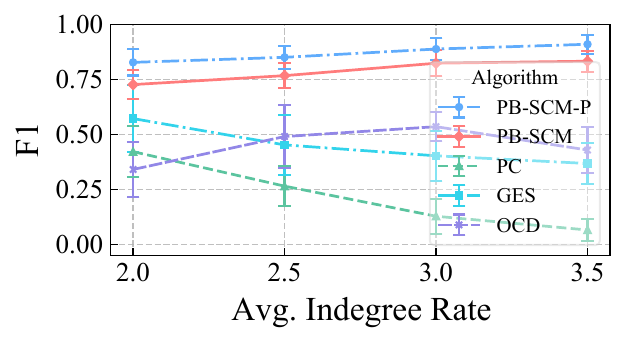}
	\label{fig:sensitivity: Indegree}
}
	\subfigure[Sensitivity to Number of vertices]{
		\includegraphics[width=0.31\textwidth]{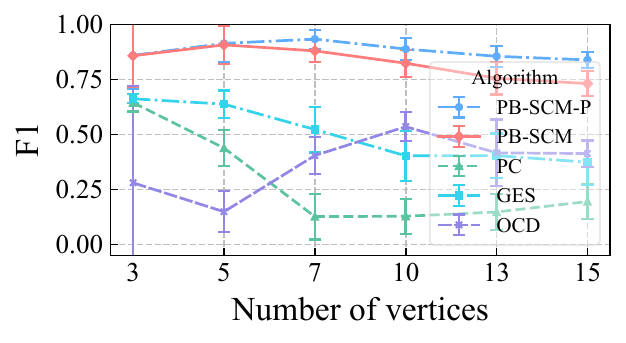}
		\label{fig:sensitivity: Node}
	}
	\subfigure[Sensitivity to Sample Size]{
	\includegraphics[width=0.32\textwidth]{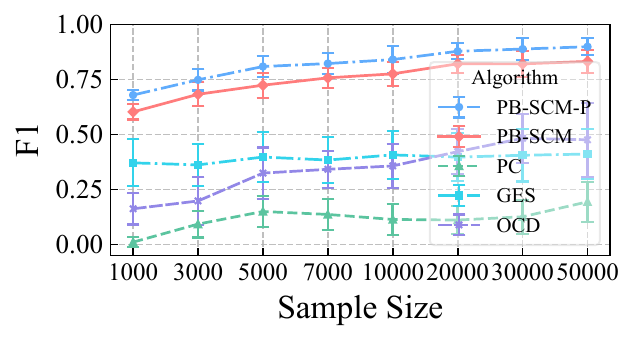}
	\label{fig:sensitivity: sample}
}
	\caption{F1 in the Sensitivity Experiments}	
	\label{fig:sensitivity}
\end{figure*}

\subsection{Synthetic Experiments}

In this section, we test the proposed PB-SCM on synthetic data. We design control experiments using synthetic data to test the sensitivity of sample size, number of vertices, and different indegree rate. The baseline methods include OCD \cite{ni2022ordinal}, PC \cite{spirtes2000causation}, GES \cite{chickering2002optimal}. We further provide the results using the true skeleton as prior knowledge (PB-SCM-P) to demonstrate the effectiveness of learning causal direction.

In the sensitivity experiment, we synthesize data with fixed parameters while traversing the target parameter as shown in Fig. \ref{fig:sensitivity}. The default settings are as follows, sample size=30000, number of vertices=10, indegree rate=3.0, range of causal coefficient $\alpha_{i,j}\in [ 0.1,0.5]$, range of the mean of Poisson noise $\mu_i \in [ 1.0,3.0]$, the max order of cumulant $K=4$. Each simulation is repeated 30 times.

As shown in Fig.\ref{fig:sensitivity}, we conduct three different control experiments for PB-SCM. Overall, our method outperforms all the baseline methods in all three control experiments.

In the control experiments of the indegree rate given in Fig. \ref{fig:sensitivity: Indegree}, as the indegree rate controls the sparse of causal structure, the higher the indegree rate, the less sparse in causal structure leading to a decrease of performance of the baseline methods. 
In contrast, PB-SCM keeps giving the best results in all indegree rates. The reason is that our method benefits from the sparsity of the graph and the denser structure would result in more causal order being identified which verified the theoretical result in our work.

In the control experiments of the number of vertices given in Fig.\ref{fig:sensitivity: Node}. Our method outperforms all the baseline methods, showing a slight decrease as the number of nodes increases, yet still demonstrating reasonable performance. The reason might be that with an increasing number of vertices, the number of paths for both directions also increases, which requires a higher-order cumulant to obtain the asymmetry. However, estimating high-order cumulant is difficult and has a large variance which leads to a decrease in performance.

In the control experiments of sample size shown in Fig.\ref{fig:sensitivity: sample}, as the sample size increases, our method's performance continues to improve and outperforms all the baseline methods. This suggests a sufficient sample size is beneficial for estimating accurate cumulant. 

\subsection{Real World Experiments}
We also test the proposed PB-SCM on a real-world football events dataset\footnote{https://www.kaggle.com/datasets/secareanualin/football-events}, which contains 941,009 events from 9,074 football games across Europe. For this experiment, we focus on the causal relation in the following count of events: Foul, Yellow card, Second yellow card (abbreviated as 2nd Y. card), Red card, and Substitution. These events possess clear causal relationships according to the rules of the football game. Our goal is to identify the causal relationship from the observed count data while reasoning the possible number of paths between two events as a byproduct of our method.

In detail, we employ the bootstrap hypothesis test with 0.05 significance level to test whether $\tilde{\Lambda}_k$ is equal to zero. The result is shown in Table \ref{exp: Real-World}. The column of $X\!\rightarrow\! Y$ shows the highest order of cumulants summation $\tilde{\Lambda}_k(X\!\leadsto\! Y)$ that is not equal to zero while the column of $Y\!\rightarrow\! X$ shows the lowest order of cumulants summation that equals zero.

\renewcommand{\arraystretch}{1.3}

\begin{table}[htp]
\centering
\scalebox{0.95}{
\begin{tabular}{l|l|lc}
\toprule
Cause $(X)$                   & Effect $(Y)$      & \multicolumn{1}{l|}{$X\rightarrow Y$}            & \multicolumn{1}{l}{$Y\rightarrow X$} \\ 
\midrule
\multirow{3}{*}{Foul}        & Yellow card     & \multicolumn{1}{l|}{$\tilde{\Lambda}_{k=2}\neq 0$}   & $\tilde{\Lambda}_{k=2}=0$              \\ \cline{2-4} 
                             & 2nd Y. card & \multicolumn{1}{l|}{$\tilde{\Lambda}_{k=3}\neq 0$} & $\tilde{\Lambda}_{k=2}=0$              \\ \cline{2-4} 
                             & Red card        & \multicolumn{1}{l|}{$\tilde{\Lambda}_{k=1}= 0$} & $\tilde{\Lambda}_{k=1}=0$           \\ \hline
\multirow{2}{*}{Yellow card} & 2nd Y. card & \multicolumn{1}{l|}{$\tilde{\Lambda}_{k=3}\neq 0$} & $\tilde{\Lambda}_{k=2}=0$              \\ \cline{2-4} 
                             & Substitution    & \multicolumn{1}{l|}{$\tilde{\Lambda}_{k=2}\neq 0$} & $\tilde{\Lambda}_{k=2}=0$              \\ \hline
2nd Y. card              & Red card        & \multicolumn{1}{l|}{$\tilde{\Lambda}_{k=2}\neq 0$}   & $\tilde{\Lambda}_{k=2}=0$              \\ 
\bottomrule
\end{tabular}
}\caption{The result of real-world dataset experiment.}
\label{exp: Real-World}

\end{table}

\begin{figure}[t]
	\centering
	\subfigure[Ground Truth]{
		\includegraphics[width=0.18\textwidth]{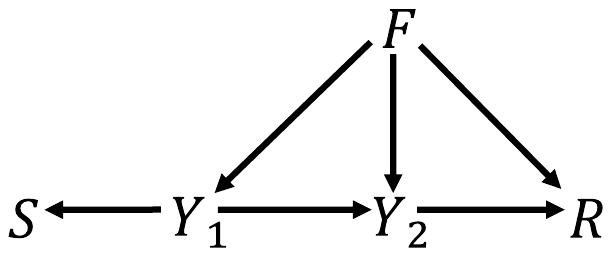}
		\label{real_data_gt}
	}
 	\subfigure[Result]{
	\includegraphics[width=0.18\textwidth]{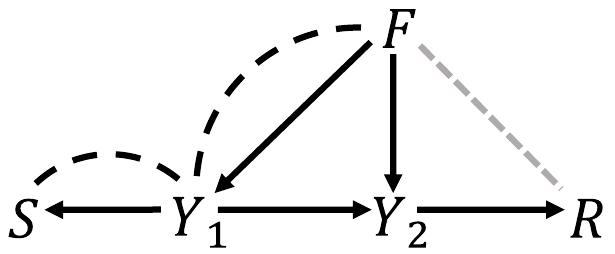}
	\label{real_data_res}
}

 \caption{Football Dataset Result ($F$:Foul, $Y_1$: Yellow card, $Y_2$: Second yellow card, $R$: Red card, $S$: Substitution).}
\end{figure}

The results are given in Fig. \ref{real_data_res}. 
Generally, PB-SCM successfully identifies five cause-effect pairs, except for $\text{Foul}\rightarrow \text{Red card}$. The possible reason might be attributed to the weak causal influence since only a few serious fouls will result in a red card. Interestingly, We find $\tilde{\Lambda} _2\left(\text{Foul}\rightarrow \text{Yellow card}\right) \neq 0$ , indicating two paths from $F$ or its ancestor to Yellow card. This suggests a hidden confounder between Foul and Yellow card, possibly related to the football team's style which also coincides with other path findings.
Moreover, the causal direction between Yellow card and Substitution is identified suggesting a hidden confounder or indirect relation exists. This result suggests the effectiveness of our method when dealing with complex real-world scenarios.

\section{Conclusion}
In this work, we study the identification of the Poisson branching structural causal model using high-order cumulant. 
We establish a link between cumulants and paths in the causal graph under PB-SCM, showing that cumulants encompass information about the number of paths between two vertices, which is retrievable. By leveraging this link, we propose the identifiability of the causal order of PB-SCM and its graphical implication. With the identifiability result, we propose a causal structure learning algorithm for PB-SCM consisting of learning causal skeleton and learning causal direction. Our theoretical results and the practical algorithm will hopefully further inspire a series of future methods to deal with count data and move the research of causal discovery further toward achieving real-world impacts in different respects.

\section{Acknowledgments}
This research was supported in part by National Key R\&D Program of China (2021ZD0111501), National Science Fund for Excellent Young Scholars (62122022), Natural Science Foundation of China (61876043, 61976052), the major key project of PCL (PCL2021A12). ZM's research was supported by the China Scholarship Council (CSC).

\bibliography{aaai24}

\newpage
\onecolumn
\setcounter{table}{0}

\appendix
\clearpage
\newtheorem{innercustomthm}{Theorem}
\newenvironment{customthm}[1]
  {\renewcommand\theinnercustomthm{#1}\innercustomthm}
  {\endinnercustomthm}

\newtheorem{innercustomcor}{Corollary}
\newenvironment{customcor}[1]
  {\renewcommand\theinnercustomcor{#1}\innercustomcor}
  {\endinnercustomcor}

  \newtheorem{innercustomlem}{Lemma}
\newenvironment{customlem}[1]
  {\renewcommand\theinnercustomlem{#1}\innercustomlem}
  {\endinnercustomlem}

  \newtheorem{innercustomprop}{Proposition}
\newenvironment{customprop}[1]
  {\renewcommand\theinnercustomprop{#1}\innercustomprop}
  {\endinnercustomprop}

  \newtheorem{innercustomremark}{Remark}
\newenvironment{customremark}[1]
  {\renewcommand\theinnercustomremark{#1}\innercustomremark}
  {\endinnercustomremark}

\numberwithin{equation}{section}
\setcounter{figure}{0}
\setcounter{remark}{0}
\setcounter{theorem}{0}
\setcounter{section}{0}
\setcounter{secnumdepth}{1}

\section*{Supplementary Material of  ``Causal Discovery from Poisson Branching Structural Causal Model Using High-Order Cumulant with Path Analysis”}

\section{Proof of Theorem 1}
\begin{theorem}[Reducibility]
\label{th: reduce}
Given a Poisson random variable $\epsilon $ and $n$ distinct sequences of coefficients $A_1 ,...,A_n$, we have
\begin{equation}
\begin{aligned}
 \kappa (\underbrace{A_1 \circ \epsilon ,...,A_1 \circ \epsilon }_{k_1 \text{times}} ,...,\underbrace{A_n \circ \epsilon ,...,A_n \circ \epsilon }_{k_n \text{times}} )
 =\kappa ( A_1 \circ \epsilon ,\dotsc , A_n \circ \epsilon )
\end{aligned}
\end{equation}
where each $A_i \circ \epsilon $ repeats $k_i$ times in the original cumulant and only contains one time in the reduced cumulant.
\end{theorem}
\subsection{Outline of Proof}
First of all, we introduce the moment-generating function (MGF) and cumulant-generating function (CGF).

\begin{definition}[Moment-generating function]
     For $X =[ X_1 ,...,X_n]^T$, an $n$-dimensional random vector, the moment-generating function of $X$ is given by:
    \begin{equation}\begin{aligned}
    M_X(\mathbf{t}) & :=E\left[ e^{\mathbf{t}^T X}\right] =E\left[ e^{t_1 X_1 +t_2 X_2 +\cdots +t_n X_n}\right]
    \end{aligned}\end{equation}
    where $\mathbf{t} =[ t_1 ,...,t_n]$ is a fixed vector.
\end{definition}

\begin{definition}[Cumulant-generating function]
    For $X =[ X_1 ,...,X_n]^T$, an $n$-dimensional random vector, the cumulant-generating function of $X$ is given by:
    \begin{equation}\begin{aligned}
    K_X(\mathbf{t}) & =\ln M_X(\mathbf{t})
    \end{aligned}\end{equation}
    where $M_X(\mathbf{t})$ is the moment-generating function of $X$.
\end{definition}

With CGF, we can calculate the joint cumulant of a given random vector $X$ by taking deviate of CGF:
\begin{equation}\begin{aligned}
\kappa ( X_1 ,X_2 ,...,X_n) & =\frac{\partial ^n K_X( t_1 ,t_2 ...,t_n)}{\partial t_1 \partial t_2 \cdots \partial t_n}\Bigl|_{t_1 =0,...,t_n =0}
\end{aligned}\end{equation}
Furthermore, if each $X_i$ in the random vector $X$ repeated $k_i$ times, then we only need to take $k_i$ times of derivatives of the CGF with respect to the corresponding $t_i$, i.e., 
\begin{equation}\begin{aligned}
\kappa (\underbrace{X_1 ,...,X_1}_{k_1\text{ times}} ,\underbrace{X_2 ,...,X_2}_{k_2\text{ times}} ,...,\underbrace{X_n ,...,X_n}_{k_n\text{ times}}) & =\frac{\partial ^{k_1 +k_2 +\cdots +k_n} K_X( t_1 ,t_2 ...,t_n)}{\partial t_1^{k_1} \partial t_2^{k_2} \cdots \partial t_n^{k_n}}\Bigl|_{t_1 =0,...,t_n =0}
\end{aligned}\end{equation}
Therefore, Theorem 1 is equivalent to show the following equality hold:
\begin{equation}\begin{aligned}\label{eq:cum_equality}
\frac{\partial ^n K_X( t_1 ,t_2 ...,t_n)}{\partial t_1 \partial t_2 \cdots \partial t_n}\Bigl|_{t_1 =0,...,t_n =0} & =\frac{\partial ^{k_1 +k_2 +\cdots +k_n} K_X( t_1 ,t_2 ...,t_n)}{\partial t_1^{k_1} \partial t_2^{k_2} \cdots \partial t_n^{k_n}}\Bigl|_{t_1 =0,...,t_n =0}.
\end{aligned}\end{equation}
To do show, we will prove that the $\frac{\partial ^n K_X( t_1 ,t_2 ...,t_n)}{\partial t_1 \partial t_2 \cdots \partial t_n}$ has the form of exponential function, i.e. 
\begin{equation}\begin{aligned}
\frac{\partial ^n K_X( t_1 ,t_2 ...,t_n)}{\partial t_1 \partial t_2 \cdots \partial t_n} & =\beta e^{t_1 +t_2 +\cdots t_n} ,
\end{aligned}\end{equation}
which is a function that remains unchanged when taking derivatives with respect to any $t_i$ and thus the Eq. \ref{eq:cum_equality} holds.

Following this outline, consider a random vector ${ \mathbf{R} =[ A_1 \circ \epsilon ,A_2 \circ \epsilon ,\dotsc ,A_n \circ \epsilon ]^T}$, $ A_i \neq A_j$, where ${ \epsilon }$ represents the Poisson noise component of a vertex ${ X}$ in graph ${ G}$, and ${ A_i}$ is a sequence of path coefficients corresponding to a direct path from ${ X}$ to one of its descendant vertices. Then according to the definition of MGF, we have:
\begin{equation}\begin{aligned}
M_{\mathbf{R}}(\mathbf{t}) & =E\left[ e^{\mathbf{t}^T\mathbf{R}}\right] =E\left[ e^{t_1 \times { A_1 \circ \epsilon +\cdots +} t_n \times { A_n \circ \epsilon }}\right]
\end{aligned}\end{equation}
where $ \mathbf{t} =[ t_1 ,t_2 ,...,t_n]^T$ is a fixed vector.

Following the outline, we first provide an intuition of proof through a specific case that each vertex are conditional independce by the root vertex $\epsilon$.

\subsection{Proof of Specific Case}
Given a Poisson random variable $\epsilon \sim Pois( \mu )$ and $n$ distinct sequences of coefficients $A_1 ,...,A_n$, in which $A_i =\left( a_k^{( i)}\right)_{k=1}^{|A_i |}$ where $a_k^{( i)}$ is the $k$-th coefficients of $A_i$.

Assume there exist no $ k=1,2,...\min( |A_i |,|A_j |)$ between any two $ A_i$ and $ A_j$ such that $ ( a_l^{( i)})_{l=1}^k =( a_l^{( j)})_{l=1}^k$, which means that there exist no two paths $ P_i$ and $ P_j$ sharing the same part from the source point. 

We consider the random vector $\mathbf{R} =( A_1 \circ \epsilon ,A_2 \circ \epsilon ,...,A_n \circ \epsilon )$, where each random variable $A_i \circ X$ appears uniquely. The moment generating function (MGF) of $\mathbf{R}$ is:
\begin{equation}\begin{aligned}
M_{\mathbf{R}}(\mathbf{t}) & =E\left[ e^{t_1 \times A_1 \circ \epsilon +\cdots +t_n \times A_n \circ \epsilon }\right].
\end{aligned}\end{equation}

According to the \textbf{law of total expectation}, we have:
\begin{equation}
\label{eq: law of E-1}
\begin{aligned}
M_{\mathbf{R}}(\mathbf{t}) & =E\left[ E\left[ e^{t_1 \times A_1 \circ \epsilon +\cdots +t_n \times A_n \circ \epsilon } |\epsilon \right]\right] =E\left[\prod _{i=1}^n E\left[ e^{t_i \times A_i \circ \epsilon}\bigl| \epsilon \right]\right],
\end{aligned}\end{equation}
since $A_i \circ \epsilon | \epsilon \perp\!\!\!\perp A_j \circ \epsilon | \epsilon$ for all $i\neq j$.

Next, according to the property of thinning operation, we have  $A_i \circ \epsilon |\epsilon \stackrel{d}{=} Binorm\left( n=\epsilon,p=\prod _{j=1}^{|A_i |} a_j^{(i)}\right)$, where `$\stackrel{d}{=}$' means distribution equality and $Binorm( n,p)$ is the binomial distribution, then the $E\left[\exp(\mathbf{t}_i \times A_i \circ \epsilon )\Bigl| \epsilon \right]$ is the MGF of a binomial variable $A_i \circ \epsilon |\epsilon $, we have:

\begin{equation}\label{eq:MGF of binomial}\begin{aligned}
E\left[\exp(\mathbf{t}_i \times A_i \circ \epsilon)\Bigl| \epsilon\right] & =\left( 1-\prod _{j=1}^{|A_i |} a_j^{( i)} +\prod _{j=1}^{|A_i |} a_j^{( i)} e^{t_i}\right)^{\epsilon }.
\end{aligned}\end{equation}

Substituting equation \ref{eq:MGF of binomial} into equation \ref{eq: law of E-1}, we have:

\begin{equation}\begin{aligned}
M_{\mathbf{R}}(\mathbf{t}) & =E\left[\prod _{i=1}^n\left( 1-\prod _{j=1}^{|A_i |} a_j^{( i)} +\prod _{j=1}^{|A_i |} a_j^{( i)} e^{t_i}\right)^{\epsilon }\right] =\sum _{k=0}^{+\infty }\left[ P( \epsilon =k)\prod _{i=1}^n\left( 1-\prod _{j=1}^{|A_i |} a_j^{( i)} +\prod _{j=1}^{|A_i |} a_j^{( i)} e^{t_i}\right)^k\right]\\
 & =\sum _{k=0}^{+\infty }\left[\frac{\mu ^k e^{-\mu }}{k!}\prod _{i=1}^n\left( 1-\prod _{j=1}^{|A_i |} a_j^{( i)} +\prod _{j=1}^{|A_i |} a_j^{( i)} e^{t_i}\right)^k\right]\\
 & =\exp( -\mu )\sum _{k=0}^{+\infty }\frac{\left[ \mu \prod _{i=1}^n\left( 1-\prod _{j=1}^{|A_i |} a_j^{( i)} +\prod _{j=1}^{|A_i |} a_j^{( i)} e^{t_i}\right)\right]^k}{k!}.
\end{aligned}\end{equation}

According to the power series expansion for the exponential function, i.e. $\exp x=\sum _{x=0}^{+\infty }\frac{x^n}{n!}$, we have:
\begin{equation}\begin{aligned}
M_{\mathbf{R}}(\mathbf{t}) & =\exp( -\mu )\exp\left[ \mu \prod _{i=1}^n\left( 1-\prod _{j=1}^{|A_i |} a_j^{( i)} +\prod _{j=1}^{|A_i |} a_j^{( i)} e^{t_i}\right)\right],
\end{aligned}\end{equation}
then the cumulant-generating function (CGF) of $\mathbf{R}$ is given by:
\begin{equation}
K_{\mathbf{R}}(\mathbf{t}) =\log M_{\mathbf{R}}(\mathbf{t}) =\mu \prod _{i=1}^n\left( 1-\prod _{j=1}^{|A_i |} a_j^{( i)} +\prod _{j=1}^{|A_i |} a_j^{( i)} e^{t_i}\right) -\mu, 
\end{equation}

We obtain the cumulant by the partial derivatives of the cumulant generating function:
\begin{equation}\begin{aligned}
\frac{\partial ^n K_{\mathbf{R}}(\mathbf{t})}{\partial t_1 \partial t_2 \cdots \partial t_n} & =\mu \prod _{i=1}^n\prod _{j=1}^{|A_i |} a_j^{( i)} e^{t_i}.
\end{aligned}\end{equation}

Since $\frac{\partial ^n K_{\mathbf{R}}(\mathbf{t})}{\partial t_1 \partial t_2 \cdots \partial t_n}$ has the form of the exponential function, further partial derivatives of it will also retain the same form:
\begin{equation}\begin{aligned}
\frac{\partial ^{k_1 +k_2 +\cdots k_n} K_{\mathbf{R}}(\mathbf{t})}{\partial t_1^{k_1} \partial t_2^{k_2} \cdots \partial t_n^{k_n}} & =\frac{\partial ^n K_{\mathbf{R}}(\mathbf{t})}{\partial t_1 \partial t_2 \cdots \partial t_n} =\mu \prod _{i=1}^n\prod _{j=1}^{|A_i |} a_j^{( i)} e^{\mathbf{t}_i}.
\end{aligned}\end{equation}

Therefore, we have:
\begin{equation}\begin{aligned}
\kappa ( A_1 \circ \epsilon ,A_2 \circ \epsilon ,...,A_n \circ \epsilon ) & =\frac{\partial ^nK_{\mathbf{R}}(\mathbf{t})}{\partial t_1 \partial t_2 \cdots \partial t_n}\Bigl|_{t_1 =0,\dotsc ,t_n =0} =\mu \prod _{i=1}^n\prod _{j=1}^{|A_i |} a_j^{( i)},\\
\kappa (\underbrace{A_1 \circ \epsilon ,...,A_1 \circ \epsilon }_{k_1 \text{times}} ,...,\underbrace{A_n \circ \epsilon ,...,A_n \circ \epsilon }_{k_n \text{times}}) & =\frac{\partial ^{k_1 +k_2 +\cdots k_n}K_{\mathbf{R}}(\mathbf{t})}{\partial t_1^{k_1} \partial t_2^{k_2} \cdots \partial t_n^{k_n}}\Bigl|_{t_1 =0,\dotsc ,t_n =0} =\mu \prod _{i=1}^n\prod _{j=1}^{|A_i |} a_j^{( i)},
\end{aligned}\end{equation}
which finishes the proof:
\begin{equation}\begin{aligned}
\kappa (\underbrace{A_1 \circ \epsilon ,...,A_1 \circ \epsilon }_{k_1 \text{times}} ,...,\underbrace{A_n \circ \epsilon ,...,A_n \circ \epsilon }_{k_n \text{times}}) & =\kappa ( A_1 \circ \epsilon ,A_2 \circ \epsilon ,...,A_n \circ \epsilon )
\end{aligned}\end{equation}

The above proof of specific case highlights that the way to calculate the MGF involves decomposing the expectation $E\left[ e^{\mathbf{t}^T\mathbf{R}}\right]$ using the law of total expectation to establish conditional independence. In the specific case, the conditional independence can be simply established by condition on the $ \epsilon $ since there is no common sub-sequence between any $ A_i$ and $ A_j$. 

However, given $ \epsilon $ is not enough to build the conditional independence if $ A_i$ and $ A_j$ share the same sub-sequence, i.e. there exist a $ k$ \ such that $ ( A_i)_{1:k} =( A_j)_{1:k}$.

Before going into the formal proof, we provide an example to illustrate why simply conditioning on $ \epsilon $ cannot establish conditional independence. Subsequently, we further show how to establish conditional independence in the presence of a common sub-sequence.

\subsection{An Example when Common sub-sequence Exists}
Consider random variables:
\begin{equation*}
A_1 \circ \epsilon =b\circ a\circ \epsilon, A_2 \circ \epsilon =c\circ a\circ \epsilon \ \text{and} \ A_3 \circ \epsilon =c\circ d\circ \epsilon ,
\end{equation*}
where $ A_1 =( a,b) ,\ A_2 =( a,c) ,\ A_3 =( d,c)$ and $ A_1$ and $ A_2$ has the common sub-sequence $ ( a)$. The generating process can be represented through a tree structure, as shown in Fig. \ref{fig: example of tree}.

\begin{figure}[h]
    \centering
    \includegraphics[width=0.2\textwidth]{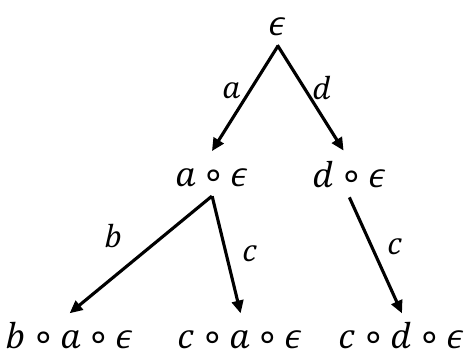}
    \caption{Generating process of $A_1 \circ \epsilon , A_2 \circ \epsilon, \ A_3 \circ \epsilon$, each leaf node corresponds to an original random variable.}
    \label{fig: example of tree}
\end{figure}

Now, if we condition on ${ \epsilon }$, we obtain conditional independence ${ A_1 \circ \epsilon |\epsilon \perp \!\!\! \perp A_3 \circ \epsilon |\epsilon }$ and ${ A_2 \circ \epsilon |\epsilon \perp \!\!\! \perp A_3 \circ \epsilon |\epsilon }$, \ however, ${ A_1 \circ \epsilon |\epsilon \not{\perp \!\!\! \perp } A_2 \circ \epsilon |\epsilon }$. This is because both ${ A_1 \circ \epsilon |\epsilon }$ and ${ A_2 \circ \epsilon |\epsilon }$ are dependent on the binomial random variable ${ a\circ \epsilon |\epsilon \stackrel{d}{=} B(n=\epsilon ,p=\alpha )}$ generated by the common sub-sequence ${ (\alpha )}$. 

Such dependence occurs due to performing the thinning operation $\circ $ on a random variable, resulting in the creation of a new random variable, distinct from the straightforward linear operations involving mere coefficient multiplication.

Therefore, to build conditional independence between ${ A_1 \circ \epsilon |\epsilon }$ and ${ A_2 \circ \epsilon |\epsilon }$, we need to further condition on ${ a\circ \epsilon |\epsilon }$. Such a process of establishing conditional independence step by step can be represented through the tree structure in Fig.\ref{fig: example of tree}, as shown in Fig. \ref{fig: example of tree 2}.

\begin{figure}[h]
    \centering
    \includegraphics[width=0.9\textwidth]{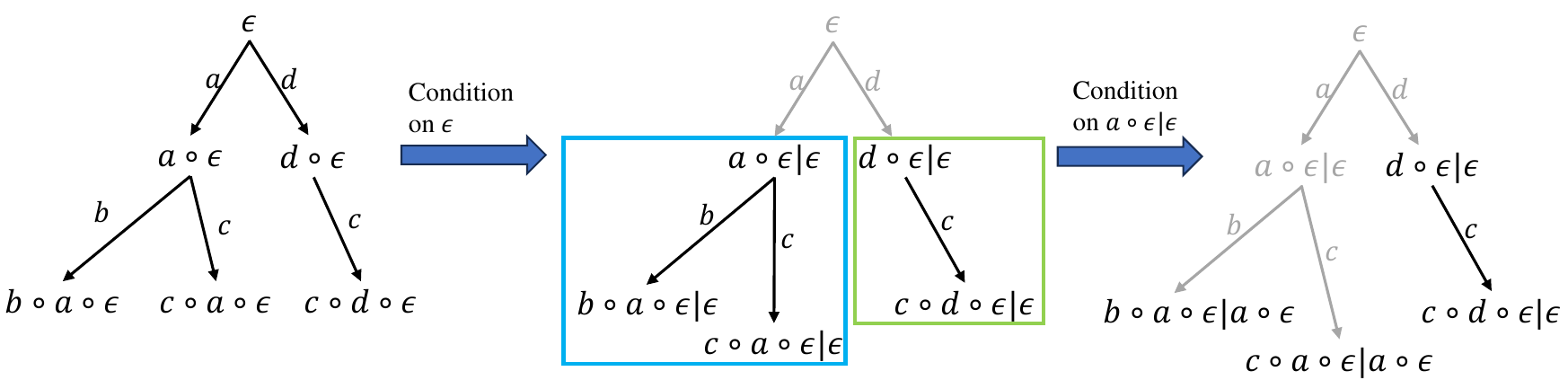}
    \caption{Obtain conditional independence according to the hierarchical structure of the tree}
    \label{fig: example of tree 2}
\end{figure}

Specifically, the MGF of $ \mathbf{R} =[{ A_1 \circ \epsilon ,A_2 \circ \epsilon ,A_3} \circ \epsilon ]^T$ is given by

\begin{equation}\begin{aligned}
M_{\mathbf{R}}( t_1 ,t_2 ,t_3) & =E\left[ e^{t_1{ \times A_1 \circ \epsilon }} e^{t_2{ \times A_2 \circ \epsilon }} e^{t_3{ \times A_3 \circ \epsilon }}\right] =E\left[ e^{t_1{ \times } b\circ a\circ \epsilon } e^{t_2{ \times } c\circ a\circ \epsilon } e^{t_3{ \times } c\circ d\circ \epsilon }\right]\\
 & =E_{\epsilon }\left[ E\left[ e^{t_1{ \times } b\circ a\circ \epsilon } e^{t_2{ \times } c\circ a\circ \epsilon } e^{t_3{ \times } c\circ d\circ \epsilon } |\epsilon \right]\right]\\
 & =E_{\epsilon }\left[ E\left[ e^{t_1{ \times } b\circ a\circ \epsilon } e^{t_2{ \times } c\circ a\circ \epsilon } |\epsilon \right] E\left[ e^{t_3{ \times } c\circ d\circ \epsilon } |\epsilon \right]\right],
\end{aligned}\end{equation}
where ${ E\left[ e^{t_1 b\circ a\circ \epsilon } e^{t_2 c\circ a\circ \epsilon } |\epsilon \right]}$ and ${ E\left[ e^{t_3 c\circ d\circ \epsilon } |\epsilon \right]}$ correspond to the blue box and the green box in Fig. \ref{fig: example of tree 2}, respectively.

The next step is to establish conditional independence and separate ${ E\left[ e^{t_1 b\circ a\circ \epsilon } e^{t_2 c\circ a\circ \epsilon } |\epsilon \right]}$. By applying the law of total expectation to it, we obtain
\begin{equation}\begin{aligned}
E\left[ e^{t_1 b\circ a\circ \epsilon } e^{t_2 c\circ a\circ \epsilon } |\epsilon \right] & =E\left[ E\left[ e^{t_1 b\circ a\circ \epsilon } |a\circ \epsilon \right] E\left[ e^{t_2 c\circ a\circ \epsilon } |a\circ \epsilon \right] |\epsilon \right] ,
\end{aligned}\end{equation}
which is calculable since $ E\left[ e^{t_1 b\circ a\circ \epsilon } |a\circ \epsilon \right]$ is the MGF of $ Binorm( n=a\circ \epsilon ,p=b)$ and the same for $ E\left[ e^{t_2 c\circ a\circ \epsilon } |a\circ \epsilon \right]$. 

Motivated by the above example, when conditioning on a vertex in a tree, conditional independence is established among the random variables corresponding to each subtree of that vertex (if the subtree exists), enabling the separation of expectations.

Therefore, one can calculate the MGF by conditioning the random variables layer by layer according to the hierarchical structure of the tree in the generating process. 

Here, to formalize the computation of the MGF, we introduce the following definition of a tree to model the generating process of the random vector $ \mathbf{R} =( A_1 \circ \epsilon ,A_2 \circ \epsilon ,...,A_n \circ \epsilon )$.

\begin{definition}[Tree representation of the generating process of random vector in PB-SCM]
For a given random vector $\mathbf{R} =[A_1 \circ \epsilon ,A_2 \circ \epsilon ,\dotsc ,A_n \circ \epsilon ]^T$, $\forall _{i,j} A_i \neq A_j$, the generating process of each random variable in $\mathbf{R}$ can be summarized by a tree $T_{\mathbf{R}}$. Let $\{T_0 ,T_1 ,T_2 ,...\}$ denote all the vertices of $T_{\mathbf{R}}$, where $T_0 =\epsilon $ is the root vertex of the tree and $T_j =\alpha _{i\rightarrow j} \circ T_i$. Let $ L=\{L_1 ,L_2 ,...,L_n\}$ with index $ i=1,2,...,n$ denote the leaf vertices in the tree, such that $ L_i =A_i \circ \epsilon $
\end{definition}

Moreover, let $A_i^j$ denotes the sub-sequence of $A_i$ that start from $T_j$. For example, $A_i=\{\alpha_{0\to 1},\alpha_{1\to 2},\alpha_{2\to 3}\}$, then $A_i^0=A_i$ and $A_i^1=\{\alpha_{1\to 2},\alpha_{2\to 3}\}$. Let $L(T_i)=\{k|T_k\text{ is leaf}\wedge k\in L(T_i)\}$ denotes the set of leaf vertex in the tree.

\subsection{The recursive relation of MGF}
Based on this definition, several lemmas are introduced to establish the recursive relation for the MGF of ${ \mathbf{R} =[A_1 \circ \epsilon ,A_2 \circ \epsilon ,\dotsc ,A_n \circ \epsilon ]^T}$.

\begin{lemma}[Start from root vertex $ T_0$]\label{le:0 to j}
    For a given random vector ${ \mathbf{R} =[ A_1 \circ \epsilon ,A_2 \circ \epsilon ,\dotsc ,A_n \circ \epsilon ]^T}$ and its tree representation $ T_{\mathbf{R}}$, $ A_i \neq A_j$, the MGF of ${ \mathbf{R}}$ satisfy
\begin{equation}\label{eq:0 to j}
\begin{aligned}
M_{\mathbf{R}} (\mathbf{t} )=E\left[ e^{\sum _{i=1}^n t_i \times A_i \circ \epsilon }\right] & =E_{T_0}\left[\prod _{j\in Ch(T_0 )} E\left[ e^{\sum _{i\in L(T_j )} t_i \times A_i^j \circ T_j} |T_0\right]\right]
\end{aligned}
\end{equation}
\end{lemma}
\begin{proof}
The result is straightforward since $T_0$ the the root of the tree, then given the condition of $T_0$ each child of $T_0$ will be conditional independence:
\begin{equation}
\begin{aligned}
M_{\mathbf{R}} (\mathbf{t} ) & =E\left[ e^{\sum _{i=1}^n t_i \times A_i \circ \epsilon }\right]\\
 & =E\left[ e^{\sum _{i=1}^n t_i \times A_i \circ T_0}\right]\\
 & =E\left[\prod _{j\in Ch(T_0 )} E\left[ e^{\sum _{i\in L(T_j )} t_i \times A_i^j \circ \alpha _{0\rightarrow j} \circ T_0} |T_0\right]\right]\\
 & =E\left[\prod _{j\in Ch(T_0 )} E\left[ e^{\sum _{i\in L(T_j )} t_i \times A_i^j \circ T_j} |T_0\right]\right]
\end{aligned}
\end{equation}
\end{proof}
By Lemma \ref{le:0 to j}, MGF can be decomposed into separated conditional expectation in the first level of the tree. Next, we will investigate how such conditional expectation can be further decomposed.

\begin{lemma}[From vertex $ T_j$ to $ T_k$]\label{le:j to k}
For a given random vector $\mathbf{R} =[A_1 \circ \epsilon ,A_2 \circ \epsilon ,\dotsc ,A_n \circ \epsilon ]^T$ and its tree representation $T_{\mathbf{R}}$. Let $\displaystyle T_j$ be a node in a level that decomposed the conditional expectation into the product of its child. Then, one of such decomposed expectation of its child $\displaystyle T_k$, can be further decomposed if $T_k$ is not leaf,
\begin{equation}
E\left[ e^{\sum _{i\in L(T_k )} t_i \times A_i^k \circ T_k} |T_j\right] =E\left[\prod _{l\in Ch(k)} E\left[ e^{\sum _{i\in L(T_l )} t_i \times A_i^l \circ T_l} |T_k\right] |T_j\right]
\end{equation}
and if $\displaystyle T_k$ is leaf,
\begin{equation}
E\left[ e^{\sum _{i\in L(T_k )} t_i \times A_i^k \circ T_k} |T_j\right] =[ M_{B(\alpha _{j\rightarrow k} )} (t_{L(T_k)} )]^{T_j} .
\end{equation}

\end{lemma}

\begin{proof}
If $T_k$ is not leaf, we can separate the expectation according its child: 
\begin{equation}
E\left[ e^{\sum _{i\in L(T_k )} t_i \times A_i^k \circ T_k} |T_j\right] =E\left[ e^{\sum _{l\in Ch( T_k)}\sum _{i\in L(T_l )} t_i \times A_i^l \circ \alpha _{k\rightarrow l} \circ T_k} |T_j\right] =E\left[\prod _{l\in Ch(k)} e^{\sum _{i\in L(T_l )} t_i \times A_i^l \circ \alpha _{k\rightarrow l} \circ T_k} |T_j\right] .
\end{equation}
Then, according to the law of total expectation, we have
\begin{equation}
\begin{aligned}
E\left[\prod _{l\in Ch(k)} e^{\sum _{i\in L(T_l )} t_i \times A_i^l \circ \alpha _{k\rightarrow l} \circ T_k} |T_j\right] & =E\left[\prod _{l\in Ch(k)} E\left[ e^{\sum _{i\in L(T_l )} t_i \times A_i^l \circ \alpha _{k\rightarrow l} \circ T_k} |T_k\right] |T_j\right]\\
 & =E\left[\prod _{l\in Ch(k)} E\left[ e^{\sum _{i\in L(T_l )} t_i \times A_i^l \circ T_l} |T_k\right] |T_j\right]
\end{aligned}
\end{equation}
If $T_k$ is leaf, which means $i=L(T_k )$ is the exactly index of the leaf vertex and $A_i^k$ is empty, and then we have: 
\begin{equation}
E\left[ e^{\sum _{i\in L(T_k )} t_i \times A_i^k \circ T_k} |T_j\right] =E\left[ e^{t_{L(T_k)} \times T_k} |T_j\right] =E\left[ e^{t_{L(T_k)} \times \alpha _{j\rightarrow k} \circ T_j} |T_j\right] .
\end{equation}
According to the definition of thin operator, we have $\alpha _{j\rightarrow k} \circ T_j =\sum\nolimits _{l=1}^{T_j} \xi _l^{(\alpha _{j\rightarrow k} )}$ with $\xi _l^{(\alpha _{j\rightarrow k} )}\stackrel{\mathrm{i.i.d.}}{\sim } B(\alpha _{j\rightarrow k} )$, where $B(\alpha _{j\rightarrow k} )$ is Bernoulli distribution with parameter $\alpha _{j\rightarrow k}$. Thus,
\begin{equation}
\begin{aligned}
E\left[ e^{t_{L(T_k)} \times \alpha _{j\rightarrow k} \circ T_j} |T_j\right] & =E\left[ e^{t_{L(T_k)} \times \sum\nolimits _{l=1}^{T_j} \xi _l^{(\alpha _{j\rightarrow k} )}} |T_j\right] =E\left[\prod\nolimits _{l=1}^{T_j} e^{t_{L(T_k)} \times \xi _l^{(\alpha _{j\rightarrow k} )}} |T_j\right] =\prod\nolimits _{l=1}^{T_j} E\left[ e^{t_{L(T_k)} \times \xi _l^{(\alpha _{j\rightarrow k} )}}\right] =E\left[ e^{t_{L(T_k)} \times \xi _l^{(\alpha _{j\rightarrow k} )}}\right]^{T_j} .
\end{aligned}
\end{equation}
Note that $E\left[ e^{t_{L(T_k)} \times \xi _l^{(\alpha _{j\to k} )}}\right]$ is the MGF of $\xi _l^{(\alpha _{j\to k} )}$. In the end, we obtain:
\begin{equation}
\begin{aligned}
E\left[ e^{t_{L(T_k)} \times \xi _l^{(\alpha _{j\rightarrow k} )}}\right]^{T_j} =[M_{B(\alpha _{j\rightarrow k} )} (t_{L(T_k)} )]^{T_j} .
\end{aligned}
\end{equation}
\end{proof}

To represent the recursive relation, we now introduce the probability-generating function (PGF):

\begin{definition}[Probability-generating function]
     For $X =[ X_1 ,...,X_n]^T$, where each $X_i$ is a discrete random variable, the probability-generating function of $X$ is given by:
    $G_X(\mathbf{z}) := E[z_1^{X_1} z_2^{X_2} \cdots  z_n^{X_n}],$
    where $\mathbf{z} =[ z_1 ,...,z_n]$.
\end{definition}

Then, following lemma disclose the recursive relation of MGF. 

\begin{lemma}\label{lem:joint MGF}
    For a given random vector $\mathbf{R} =[A_1 \circ \epsilon ,A_2 \circ \epsilon ,\dotsc ,A_n \circ \epsilon ]^T$ and its tree representation $T_{\mathbf{R}}$. Let $\displaystyle M_{j,k}(\mathbf{t}) :=E\left[ e^{\sum _{i\in L(T_k )} t_i \times A_i^k \circ T_k} |T_j\right]$ and $\displaystyle \tilde{M}_{j,k}(\mathbf{t}_{L(T_k)}) =[ M_{j,k}(\mathbf{t}_{L(T_k)})]^{\frac{1}{T_j}}$, where $\textbf{t}_{L(T_k)}=\{t_i|i\in L(T_k)\}$. The joint MGF can be expressed as follows:
\begin{equation}
M_{\mathbf{R}} (\mathbf{t} )=G_{T_0}\left(\prod _{j\in Ch(T_0 )}\tilde{M}_{0,j}(\mathbf{t}_{L(T_j)})\right) ,
\end{equation}
where
\begin{equation}
\tilde{M}_{j,k}(\mathbf{t}_{L(T_k)}) =\begin{cases}
G_{B( \alpha _{j\rightarrow k})}\left(\prod _{l\in Ch(k)}\tilde{M}_{k,l}(\mathbf{t}_{L(T_k)})\right) & \text{if } T_k\text{ is not leaf vertex}\\
M_{B(\alpha _{j\rightarrow k} )} (t_{L( T_k)} ) & \text{otherwise}
\end{cases} .
\end{equation}
\end{lemma}

\begin{proof}
First, by Lemma 2, we have the following recursive formula
\begin{equation}\label{eq:recusive}
M_{j,k}(\mathbf{t}_{L(T_k)}) =\begin{cases}
E\left[\prod _{l\in Ch(k)} M_{k,l}(\mathbf{t}_{L(T_l)}) |T_j\right] & \text{if } T_k\text{ is not leaf vertex}\\
[ M_{B(\alpha _{j\rightarrow k} )} (t_i )]^{T_j} & \text{otherwise}
\end{cases}
\end{equation}
and thus 
\begin{equation}
M_{\mathbf{R}} (\mathbf{t} )=E_{T_0}\left[\prod _{j\in Ch(T_0 )} M_{0,j}(\mathbf{t}_{L(T_j)})\right] .
\end{equation}
Then, since $\displaystyle \tilde{M}_{j,k}(\mathbf{t}_{L(T_k)}) =[ M_{j,k}(\mathbf{t}_{L(T_k)})]^{\frac{1}{T_j}}$, based on the recursive formula in Eq. \ref{eq:recusive}, we have
\begin{equation}\label{eq:recursive2}
\begin{aligned}
\tilde{M}_{j,k}(\mathbf{t}_{L(T_k)}) & =M_{j,k}(\mathbf{t}_{L(T_k)})^{\frac{1}{T_j}}\\
 & =\begin{cases}
\left( E\left[\prod _{l\in Ch(k)}\tilde{M}_{k,l}(\mathbf{t}_{L(T_l)})^{T_k} |T_j\right]\right)^{\frac{1}{T_j}} & \text{if } T_k\text{ is not leaf vertex}\\
M_{B(\alpha _{j\rightarrow k} )} (t_i ) & \text{otherwise}
\end{cases}\\
 & =\begin{cases}
\left( E\left[\left(\prod _{l\in Ch(k)}\tilde{M}_{k,l}(\mathbf{t}_{L(T_l)})\right)^{\alpha _{j\rightarrow k} \circ T_j} |T_j\right]\right)^{\frac{1}{T_j}} & \text{if } T_k\text{ is not leaf vertex}\\
M_{B(\alpha _{j\rightarrow k} )} (t_{L(T_k)} ) & \text{otherwise}
\end{cases}
\end{aligned}
\end{equation}
Consider the expectation in Eq. \ref{eq:recursive2} when $\displaystyle T_k$ is not leaf vertex. Since $\displaystyle T_j$ is conditioned such that $\displaystyle \alpha _{j\rightarrow k} \circ T_j$ follows the distribution $\displaystyle Binorm( n=T_j ,p=\alpha _{j\rightarrow k})$, then by the probability generating function of Binomial distribution, we have
\begin{equation}\label{eq:recursive3}
E\left[\left(\prod _{l\in Ch(k)}\tilde{M}_{k,l}(\mathbf{t}_{L(T_l)})\right)^{\alpha _{j\rightarrow k} \circ T_j} |T_j\right] =G_{B( \alpha _{j\rightarrow k})}\left(\prod _{l\in Ch(k)}\tilde{M}_{k,l}(\mathbf{t}_{L(T_l)})\right)^{T_j} .
\end{equation}
where $\displaystyle G_{B( \alpha _{j\rightarrow k})}( \cdot )$ is the probability generating function of Bernoulli distribution according to the relation between Bernoulli and Binomial distribution. Substituting Eq. \ref{eq:recursive3} into Eq. \ref{eq:recursive2} we have
\begin{equation}
\tilde{M}_{j,k}(\mathbf{t}_{L(T_k)}) =\begin{cases}
G_{B( \alpha _{j\rightarrow k})}\left(\prod _{l\in Ch(k)}\tilde{M}_{k,l}(\mathbf{t}_{L(T_l)})\right) & \text{if } T_k\text{ is not leaf vertex}\\
M_{B(\alpha _{j\rightarrow k} )} (t_i ) & \text{otherwise}
\end{cases}.
\end{equation}
As for the joint MGF, similarly, we have
\begin{equation}
M_{\mathbf{R}} (\mathbf{t} )=E_{T_0}\left[\prod _{j\in Ch(T_0 )} M_{0,j}(\mathbf{t}_{L(T_j)})\right] =E_{T_0}\left[\left(\prod _{j\in Ch(T_0 )}\tilde{M}_{0,j}(\mathbf{t}_{L(T_j)})\right)^{T_0}\right] =G_{T_0}\left(\prod _{j\in Ch(T_0 )}\tilde{M}_{0,j}(\mathbf{t}_{L(T_j)})\right) ,
\end{equation}
which finishes the proof.

\end{proof}

After deriving the recursive relation, we now step into the formal proof of Theorem 1 following the proof outline. 

\subsection{Formal Proof of Theorem 1}

According to the Lemma \ref{lem:joint MGF}, for a given random vector $\mathbf{R} =[A_1 \circ \epsilon ,A_2 \circ \epsilon ,\dotsc ,A_n \circ \epsilon ]^T$ and its tree representation $T_{\mathbf{R}}$, the MGF of it is $\displaystyle M_{\mathbf{R}} (\mathbf{t} )=G_{T_0}\left(\prod\nolimits _{j\in Ch(T_0 )}\tilde{M}_{0,j} (\mathbf{t}_{L(T_j)} )\right)$, where $\displaystyle T_0 =\epsilon \sim \text{Pois}( \mu )$ is a Poisson random variable. The cumulant generating function of $\mathbf{R}$ is given by: 
\begin{equation}
\begin{aligned}
K_{\mathbf{R}} (\mathbf{t} )=\log M_{\mathbf{R}} (\mathbf{t} ) & =\log\exp\left[ \mu \left(\prod\nolimits _{j\in Ch(T_0 )}\tilde{M}_{0,j} (\mathbf{t}_{L(T_j)} )-1\right)\right] =\mu \left(\prod\nolimits _{j\in Ch(T_0 )}\tilde{M}_{0,j} (\mathbf{t}_{L(T_j)} )-1\right) .
\end{aligned}
\end{equation}

Our goal is to show the derivative of $K_{\mathbf{R}} (\mathbf{t} )$ is of the exponential form:
\begin{equation}
\frac{\partial ^n K_{\mathbf{R}} (\mathbf{t} )}{\partial t_1 \partial t_1 \cdots \partial t_n} =\beta e^{t_1 +t_2 +\cdots +t_n}.
\end{equation}


We start with

\begin{equation}
\frac{\partial ^n K_{\mathbf{R}} (\mathbf{t} )}{\partial t_1 \partial t_1 \cdots \partial t_n} =\frac{\partial ^n}{\partial t_1 \partial t_1 \cdots \partial t_n} \mu \left(\prod\limits _{j\in Ch(T_0 )}\tilde{M}_{0,j} (\mathbf{t}_{L( T_j)} )-1\right) =\mu \frac{\partial ^n}{\partial t_1 \partial t_1 \cdots \partial t_n}\prod\limits _{j\in Ch(T_0 )}\tilde{M}_{0,j} (\mathbf{t}_{L( T_j)}),
\end{equation}
where $\tilde{M}_{0,j} (\mathbf{t}_{L( T_j)})$ is a function involving only $\{t_i|i \in L(T_j)\}$, we then have:

\begin{equation}
 \mu \frac{\partial ^n}{\partial t_1 \partial t_1 \cdots \partial t_n}\prod\limits _{j\in Ch(T_0 )}\tilde{M}_{0,j} (\mathbf{t}_{L( T_j)})=\mu \prod\nolimits _{j\in Ch(T_0 )}\frac{\partial ^{|L( T_j) |}}{\prod _{i \in L( T_j)} \partial t_i}\tilde{M}_{0,j} (\mathbf{t}_{L( T_j)} ).
\end{equation}

We then introduce the recursive representation of $\frac{\partial ^{|L( T_j) |}}{\prod _{i \in L( T_j)} \partial t_i}\tilde{M}_{0,j} (\mathbf{t}_{L( T_j)} ).$

\begin{lemma}
\label{lem:derivative}
The higher order partial derivative of $\tilde{M}_{j,k} (\mathbf{t}_{L( T_k)} )$ can be given by:

\begin{equation}
    \begin{aligned}
    \frac{\partial ^{|L( T_k) |}\tilde{M}_{j,k} (\mathbf{t}_{L( T_k)} )}{\prod _{i \in L( T_k)} \partial t_i} & =\begin{cases}
    \alpha _{j\rightarrow k}\prod _{l\in Ch(k)}\frac{\partial ^{|L( T_l) |} \tilde{M}_{k,l} (\mathbf{t}_{L( T_l)} )}{\prod _{t_i \in L( T_l)} \partial t_i}, & \text{if} \ T_k \ \text{is not leaf vertex},\\
    \alpha _{j\rightarrow k} e^{t_{L(T_k)}}, & \text{otherwise}.
    \end{cases}
    \end{aligned}
\end{equation}
\end{lemma}

\begin{proof}
When $\displaystyle T_k$ is not a leaf vertex, according to the Lemma \ref{lem:joint MGF}, we have:
\begin{equation}
    \begin{aligned}
        \frac{\partial ^{|L( T_k) |}\tilde{M}_{j,k} (\mathbf{t}_{L( T_k)} )}{\prod _{i \in L( T_k)} \partial t_i} & =\frac{\partial ^{|L( T_k) |}}{\prod _{i \in L( T_k)} \partial t_i} G_{B(\alpha _{j\rightarrow k} )}\left(\prod _{l\in Ch(k)}\tilde{M}_{k,l} (\mathbf{t}_{L( T_l)} )\right)\\
 & =\frac{\partial ^{|L( T_k) |}}{\prod _{i \in L( T_k)} \partial t_i}\left( 1-\alpha _{j\rightarrow k} +\alpha _{j\rightarrow k}\prod _{l\in Ch(k)}\tilde{M}_{k,l} (\mathbf{t}_{L( T_l)} )\right) \\
 &=\alpha _{j\rightarrow k}\frac{\partial ^{|L( T_k) |}}{\prod _{i \in L( T_k)} \partial t_i}\prod _{l\in Ch(k)}\tilde{M}_{k,l} (\mathbf{t}_{L( T_l)} ).
    \end{aligned}
\end{equation}
Since $\tilde{M}_{k,l} (\mathbf{t}_{L( T_l)})$ is a function involving only $\{t_i|i \in L(T_l)\}$, we have:
\begin{equation}
\alpha _{j\rightarrow k}\frac{\partial ^{|L( T_k) |}}{\prod _{i \in L( T_k)} \partial t_i}\prod _{l\in Ch(k)}\tilde{M}_{k,l} (\mathbf{t}_{L( T_l)} ) =\alpha _{j\rightarrow k}\prod _{l\in Ch(k)}\frac{\partial ^{|L( T_l) |}\tilde{M}_{k,l} (\mathbf{t}_{L( T_l)} )}{\prod _{i \in L( T_k)} \partial t_i}.
\end{equation}

Otherwise, when $T_k$ is a leaf vertex, we have:
\begin{equation}
\frac{\partial ^{|L( T_k) |}\tilde{M}_{j,k} (\mathbf{t}_{L( T_k)} )}{\prod _{i \in L( T_k)} \partial t_i} =\frac{\partial M_{B(\alpha _{j\rightarrow k} )} (t_{L(T_k)} )}{\partial t_{L(T_k)}} =\frac{\partial \left( 1-\alpha _{j\rightarrow k} +\alpha _{j\rightarrow k} e^{t_{L(T_k)}}\right)}{\partial t_{L(T_k)}} =\alpha _{j\rightarrow k} e^{t_{L(T_k)}},
\end{equation}
which finishes the proof.
\end{proof}

According to Lemma \ref{lem:derivative}, as the recursion terminates with an exponential function upon reaching the leaf vertex, we can deduce that the expansion of $\frac{\partial ^n K_{\mathbf{R}} (\mathbf{t} )}{\partial t_1 \partial t_1 \cdots \partial t_n}$ results in the product of $e^{t_i}$ for all $i \in [n]$, along with a series of corresponding path coefficients. Moreover, our focus does not lie in the specific form of these coefficients, and thus we denote the coefficient as $\beta$. We conclude:

\begin{equation}
\frac{\partial ^n K_{\mathbf{R}} (\mathbf{t} )}{\partial t_1 \partial t_1 \cdots \partial t_n} =\beta e^{t_1 +t_2 +\cdots +t_n}.
\end{equation}

Finally, we obtain:
\begin{equation}
\begin{aligned}
\kappa (A_1 \circ \epsilon ,A_2 \circ \epsilon ,...,A_n \circ \epsilon ) & =\frac{\partial ^n K_{\mathbf{R}} (\mathbf{t} )}{\partial t_1 \partial t_2 \cdots \partial t_n}\Bigl|_{t_1 =0,\dotsc ,t_n =0} =\beta,\\
\kappa (\underbrace{A_1 \circ \epsilon ,...,A_1 \circ \epsilon }_{k_1\text{times}} ,...,\underbrace{A_n \circ \epsilon ,...,A_n \circ \epsilon }_{k_n\text{times}} ) & =\frac{\partial ^{k_1 +k_2 +\cdots k_n} K_{\mathbf{R}} (\mathbf{t} )}{\partial t_1^{k_1} \partial t_2^{k_2} \cdots \partial t_n^{k_n}}\Bigl|_{t_1 =0,\dotsc ,t_n =0} =\beta,
\end{aligned}
\end{equation}
which finishes the proof:
\begin{equation}
\begin{aligned}
\kappa (\underbrace{A_1 \circ \epsilon ,...,A_1 \circ \epsilon }_{k_1\text{times}} ,...,\underbrace{A_n \circ \epsilon ,...,A_n \circ \epsilon }_{k_n\text{times}} ) & =\kappa (A_1 \circ \epsilon ,A_2 \circ \epsilon ,...,A_n \circ \epsilon ).
\end{aligned}
\end{equation}

\section{Proof of Remark 1}
\begin{remark}
    For any two variables causal graph, the causal direction of PB-SCM is not identifiable and a distributed equivalent reversed model exists.
\end{remark}
\begin{proof}
    We prove by the equality of PGFs of both directions.
    Given a two variables causal graph $X\xrightarrow{\alpha} Y$, where $X=\epsilon_X, Y=\alpha \circ X + \epsilon_Y, \epsilon_i\sim \text{Pois}(\mu_i)$, and we denote the reverse model by $Y\xrightarrow{\hat{\alpha}}X$, where $Y=\hat{\epsilon_Y}, X=\hat\alpha \circ Y + \hat{\epsilon_X}$. We now show the solution of $\hat{\alpha}, \hat{\epsilon_X}$ and $\hat{\epsilon_Y}$.

    For the causal direction, the PGF is given by:
    \begin{equation}
    \begin{aligned}
    G_{X,Y}( z_1 ,s_2) & =E\left[ z_1^X z_2^{\alpha \circ X+\epsilon _Y}\right]\\
     & =E\left[ z_1^{\epsilon _X} z_2^{\alpha \circ \epsilon _X}\right] E\left[ z_2^{\epsilon _Y}\right]\\
     & =G_{\epsilon _X}( z_1 G_{B( \alpha )}( z_2)) G_{\epsilon _Y}( z_2)\\
     & =G_{\epsilon _X}( z_1( 1-\alpha +\alpha z_2)) G_{\epsilon _Y}( z_2)\\
     & =e^{\mu _X( z_1( 1-\alpha +\alpha z_2) -1)} e^{\mu _Y( z_2 -1)}.
    \end{aligned}
    \end{equation}

    For the reverse direction, the PGF is given by:

    \begin{equation}
    \begin{aligned}
\hat{G}_{X,Y}( z_{1} ,z_{2}) & =E\left[ z_{1}^{\hat{\alpha } \circ Y+\hat{\epsilon }_{X}} z_{2}^{Y}\right]\\
 & =E\left[ z_{1}^{\hat{\alpha } \circ Y} z_{2}^{Y}\right] E\left[ z_{1}^{\hat{\epsilon }_{X}}\right]\\
 & =G_{Y}( z_{2} G_{B(\hat{\alpha })}( z_{1})) G_{\hat{\epsilon }_{X}}( z_{1})\\
 & =G_{Y}( z_{2}( 1-\hat{\alpha } +\hat{\alpha } z_{1})) G_{\hat{\epsilon }_{X}}( z_{1})\\
 & =e^{E[ Y]( z_{2}( 1-\hat{\alpha } +\hat{\alpha } z_{1}) -1)} e^{\hat{\mu }_{x}( z_{1} -1)}.
    \end{aligned}
    \end{equation}

    If these two models are equivalent, we have 
$G_{X,Y}( z_{1} ,z_{2})  =\hat{G}_{X,Y}( z_{1} ,z_{2})$, i.e.
    \begin{equation}
    \mu _{X}( z_{1}( 1-\alpha +\alpha z_{2}) -1) +\mu _{Y}( z_{2} -1) =E[ Y]( z_{2}( 1-\hat{\alpha } +\hat{\alpha } z_{1}) -1) +\hat{\mu }_{X}( z_{1} -1).
    \end{equation}

    As $Y$ is a root vertex in the reverse model, we have $\epsilon_Y\sim \text{Pois}(E[Y])=\text{Pois}(\alpha \mu _X +\mu _Y)$. Then we have:
\begin{equation}
    \mu _{X}( z_{1}( 1-\alpha +\alpha z_{2}) -1) +\mu _{Y}( z_{2} -1) =( \alpha \mu _{x} +\mu _{y})( z_{2}( 1-\hat{\alpha } +\hat{\alpha } z_{1}) -1) +\hat{\mu }_{X}( z_{1} -1).
\end{equation}

    Expanding the expression and simplifying, we obtain
\begin{equation}
    \begin{aligned}
     & \alpha \mu_{X} z_{1} z_{2} +\mu _{X}( 1-\alpha ) z_{1} +\mu_{Y} z_{2} -\mu _{X} -\mu _{Y}\\
     & =( \alpha \mu_{X} +\mu _{Y})\hat{\alpha } z_{1} z_{2} +\hat{\mu }_{X} z_{1} +( \alpha \mu _{X} +\mu _{Y})( 1-\hat{\alpha }) z_{2} -( \alpha \mu _{X} +\mu _{Y}) -\hat{\mu }_{X}
    \end{aligned}
    \end{equation}

    To ensure the equality holds, we equate the coefficients, resulting in following system of equations:
\begin{equation}
    \begin{aligned}
    \alpha \mu _{X} & =( \alpha \mu _{X} +\mu _{Y})\hat{\alpha }\\
    \mu _{X}( 1-\alpha ) & =\hat{\mu }_{X}\\
    \mu _{Y} & =( \alpha \mu _{X} +\mu _{Y})( 1-\hat{\alpha })\\
    \mu _{X} +\mu _{Y} & =( \alpha \mu _{X} +\mu _{Y}) +\hat{\mu }_{X},
    \end{aligned}
    \end{equation}
    where the solution of it is $\hat{\alpha } =\alpha \mu _{X} /( \alpha \mu _{X} +\mu _{Y}) ,\ \hat{\mu }_{X} =\mu _{X}( 1-\alpha )$. This completes the proof.
    
\end{proof}

\section{Proof of Theorem 2}
\begin{theorem}\label{th:lambda and cumulant root}
For any two vertex $i$ and $j$ where $i$ is root vertex, i.e., vertex $i$ has empty parent set, the 2D slice of joint cumulant $\mathcal{C}_{i,j}^{(n)}$ satisfies:
\begin{equation}\label{eq:cumulant_lambda_root}
\begin{aligned}
\mathcal{C}_{i,j}^{(n)} & =\sum\nolimits _{k=1}^{n-1}\!\!\!\!\sum\limits _{\overset{m_1 +\cdots +m_k =n-1}{m_l  >0}}\!\!\binom{n-1}{m_1 m_2 \cdots m_k} \Lambda_k^{i \leadsto j}( 1\!\circ\! X_i \!\leadsto\! X_j).
\end{aligned}
\end{equation}
where $\binom{n-1}{m_1 m_2 \cdots m_k} =\frac{(n-1)!}{m_1 !m_2 !\cdots m_k !}$ is the multinomial coefficients.
\end{theorem}

\subsection{Proof}
For any two vertex $i$ and $j$, where $i$ is root vertex, let $\mathbf{P}^{i\leadsto j} =\left\{P_1^{i\leadsto j} ,P_2^{i\leadsto j} ,...,P_{|\mathbf{P}^{i\leadsto j} |}^{i\leadsto j}\right\}$ be the set of paths from vertex $i$ to $j$ with the corresponding set of sequences of coefficients$\mathbf{A}^{i\leadsto j} =\left\{A_1^{i\leadsto j} ,A_2^{i\leadsto j} ,...,A_{|\mathbf{P}^{i\leadsto j} |}^{i\leadsto j}\right\}$. According to the definition of $\mathcal{C}_{i,j}^{(n)}$, we have:

\begin{equation}\begin{aligned}
\mathcal{C}_{i,j}^{(n)} & =\kappa \left( X_i ,\underbrace{X_j ,\dotsc ,X_j}_{n-1\text{times}}\right) =\kappa \left( \epsilon _i ,\underbrace{X_j ,\dotsc ,X_j}_{n-1\text{times}}\right)
\end{aligned}\end{equation}

then we expand $X_j$ according to the structural equation of $X_j$:

\begin{equation}
\label{eq: Cij_expand}
\begin{aligned}
\mathcal{C}_{i,j}^{(n)} & =\kappa \Big( \epsilon _i ,\underbrace{\underbrace{A_1^{i\leadsto j} \circ \epsilon _i +\cdots +A_{|\mathbf{P}^{i\leadsto j} |}^{i\leadsto j} \circ \epsilon _i}_{|\mathbf{P}^{i\leadsto j} |\text{times}} ,\dotsc ,A_1^{i\leadsto j} \circ \epsilon _i +\cdots +A_{|\mathbf{P}^{i\leadsto j} |}^{i\leadsto j} \circ \epsilon _i}_{n-1\text{times}}\Big) .
\end{aligned}\end{equation}
By applying the multilinearity of cumulant, we obtain the following decomposition:
\begin{equation}\label{eq:cumulant_full_root}
\begin{aligned}
\mathcal{C}_{i,j}^{(n)} & =\sum _{l_1 =1}^{|\mathbf{P}^{i\leadsto j} |} \cdots \sum _{l_{n-1} =1}^{|\mathbf{P}^{i\leadsto j} |} \kappa \left( \epsilon _i ,A_{l_1}^{i\leadsto j} \circ \epsilon _i ,A_{l_2}^{i\leadsto j} \circ \epsilon _i ,\dotsc ,A_{l_{n-1}}^{i\leadsto j} \circ \epsilon _i\right) ,
\end{aligned}\end{equation}
which yields $|\mathbf{P}^{i\leadsto j} |\times ( n-1)$ cumulant term where each term correspondent to a different combination of coefficient $A_l^{i\leadsto j}$. 

To characterize the combinations of $A_l^{i\leadsto j}$ within each cumulant in Eq. \ref{eq: Cij_expand}, we can conceptualize that choose different $A_l^{i\leadsto j}$ into $n-1$ box from $|\mathbf{P}^{i\leadsto j} |$ number of different coefficient, , i.e., we can select a $A_l^{i\leadsto j}$ from $\mathbf{A}^{i\leadsto j} =\left\{A_1^{i\leadsto j} ,A_2^{i\leadsto j} ,...,A_{|\mathbf{P}^{i\leadsto j} |}^{i\leadsto j}\right\}$ for each position in the decomposed cumulant.

To establish the connection between the cumulant and the $k$-path summation, we first recall the definition of $\Lambda _k^{i\leadsto j} (1\circ X_i \leadsto X_j )$:
\begin{equation}
\label{eq: definition of lambda}
\begin{aligned}
\Lambda _k^{i\leadsto j} (1\circ \epsilon _i \leadsto X_j ) & =\sum _{1\leq l_1 < l_2 < ...< l_k \leq |\mathbf{P}^{i\leadsto j} |} \kappa (1\circ \epsilon _i ,A_{l_1}^{i\leadsto j} \circ \epsilon _i ,...,A_{l_k}^{i\leadsto j} \circ \epsilon _i ),
\end{aligned}
\end{equation}
which is the sum of cumulants and each cumulant involve $k$ distinct $A_l^{i\leadsto j}$. 

Note that due to the reducibility, the $\mathcal{C}_{i,j}^{(n)}$ can be reduced to several distinct cumulants. In particular, the $k$-path summation contains all the distinct cumulants of $\mathcal{C}_{i,j}^{(n)}$ which involve $k$ distinct $A_l^{i\leadsto j}$. Therefore, the connection between Eq. \ref{eq: definition of lambda} and Eq. \ref{eq:cumulant_full_root} can be formulated as how many numbers for each distinct $A_l^{i\leadsto j}$ occurs after the reducing.

For each distinct cumulant in $\Lambda _k^{i\leadsto j} (1\circ \epsilon _i \leadsto X_j )$, the number of occurrences is the same because each cumulant is constructed by $k$ different paths sharing the same property in term of number.

Thus, we only need to count the number of each distinct cumulant for some specific $k$ paths. Without loss of generality, consider the cumulant with $k$ path information: $\kappa (1\circ \epsilon _i ,A_1^{i\leadsto j} \circ \epsilon _i ,...,A_k^{i\leadsto j} \circ \epsilon _i )$. Since before the reducing step, there are $n-1$ positions for each $A$, and the count can be formulated by counting the number of ways to place $k$ distinguishable $A$ into $n-1$ indistinguishable boxes with replacement such that each ball must appear at least once. Such a number can be calculated by $\sum\limits _{\overset{m_1 +\cdots +m_k =n-1}{m_l  >0}}\binom{n-1}{m_1 m_2 \cdots m_k}$, which is the coefficient of $\Lambda _k^{i\leadsto j} (1\circ \epsilon _i \leadsto X_j )$. By combining each order of $k$, we can obtain the close-form solution of $\mathcal{C}_{i,j}^{(n)}$ in Eq. \ref{eq: Cij_expand}. This completes the proof.

\section{Proof of Theorem 3 and Theorem 4}

\begin{theorem}[Identifiability for root vertex]\label{thm:identifiability root}
    For any vertex $i$ and $j$, where $i$ is the root vertex in graph $G$, if $\mathcal{C}_{i,j}^{(3)}-\mathcal{C}_{i,j}^{(2)}\ne 0$, then $\mathcal{C}_{j,i}^{(3)}-\mathcal{C}_{j,i}^{(2)}=0$ and $X_i$ is the ancestor of $X_j$.
\end{theorem}

\begin{proof}
        For the reverse direction, since $X_i$ is a root vertex, we have:
    \begin{equation}
    \label{eq: Cji3-Cji2}
    \begin{aligned}
    \mathcal{C}_{j,i}^{(2)} & =\kappa ( X_j ,X_i) =\kappa \left(\sum _{l=1}^{|\mathbf{P}^{i\leadsto j} |} A_l^{i\leadsto j} \circ \epsilon _i ,\epsilon _i\right) ,\\
    \mathcal{C}_{j,i}^{(3)} & =\kappa ( X_j ,X_i ,X_i) =\kappa \left(\sum _{l=1}^{|\mathbf{P}^{i\leadsto j} |} A_l^{i\leadsto j} \circ \epsilon _i ,\epsilon _i ,\epsilon _i\right) .
    \end{aligned}\end{equation}
    Based on Theorem 1, Eq. \ref{eq: Cji3-Cji2} can be reduced as follow:
    \begin{equation}\begin{aligned}
    \mathcal{C}_{j,i}^{(3)} & =\kappa \left(\sum _{l=1}^{|\mathbf{P}^{i\leadsto j} |} A_l^{i\leadsto j} \circ \epsilon _i ,\epsilon _i ,\epsilon _i\right) =\kappa \left(\sum _{l=1}^{|\mathbf{P}^{i\leadsto j} |} A_l^{i\leadsto j} \circ \epsilon _i ,\epsilon _i\right) =\mathcal{C}_{j,i}^{(2)}
    \end{aligned}\end{equation}
    thus $ \mathcal{C}_{j,i}^{(3)} -\mathcal{C}_{j,i}^{(2)} =0$.
    
    For the causal direction, based on Theorem 2, we have:
    \begin{equation}\begin{aligned}
    \mathcal{C}_{i,j}^{(2)} & =\Lambda _1^{i\leadsto j} (X_i \leadsto X_j )\\
    \mathcal{C}_{i,j}^{(3)} & =\Lambda _1^{i\leadsto j} (X_i \leadsto X_j )+2\Lambda _2^{i\leadsto j} (X_i \leadsto X_j )
    \end{aligned}\end{equation}
    Then we have $
    \mathcal{C}_{i,j}^{(3)} -\mathcal{C}_{i,j}^{(2)} =2\Lambda _2^{i\leadsto j} (X_i \leadsto X_j )\neq 0
    $ which means that there are more than one path from $ i$ to $ j$, i.e. $|\mathbf{P}^{i\leadsto j} |\geq 2$. Therefore, $X_i$ is the ancestor of $X_j$. This completes the proof.
\end{proof}

\begin{theorem}[Graphical Implication of Identifiability for Root Vertex]
    For a pair of vertices $i$ and $j$ in graph $G$, if vertex $i$ is a root vertex and exists at least two directed paths from $i$ to $j$, i.e., $|\mathbf{P}^{i\leadsto j}|\geq 2$ then the causal order between $i$ and $ j$ is identifiable. 
\end{theorem}

\begin{proof}
Suppose the causal order between $i$ and $j$ can not be identified by Theorem \ref{thm:identifiability root}. Then there must be in the following cases: (i) $\mathcal{C}_{i,j}^{(3)}-\mathcal{C}_{i,j}^{(2)}=0$; (ii) $\mathcal{C}_{j,i}^{(3)}-\mathcal{C}_{j,i}^{(2)}\ne0$.

For (i), since $\Lambda _2^{i\leadsto j} (1\circ \epsilon _i \leadsto X_j )= \mathcal{C}_{i,j}^{(3)}-\mathcal{C}_{i,j}^{(2)}=0$ indicating that there exists zero or one path from $i$ to $j$ which contradict to the fact that $|\mathbf{P}^{i\leadsto j}|\geq 2$.

For (ii), $\mathcal{C}_{j,i}^{(3)}-\mathcal{C}_{j,i}^{(2)}\ne0$ is contradicted to Theorem \ref{thm:identifiability root} that  $\mathcal{C}_{j,i}^{(3)}-\mathcal{C}_{j,i}^{(2)}=0$ when $\mathcal{C}_{i,j}^{(3)}-\mathcal{C}_{i,j}^{(2)}\ne 0$. By combining these two cases, we complete the proof.
\end{proof}

\section{Proof of Theorem 5}
\begin{theorem}
    For any two vertex $i$ and $j$, the 2D slice of joint cumulant $\mathcal{C}_{i,j}^{(n)}$ satisfies:
\begin{equation}\label{eq:cumulant_lambda}
    \begin{aligned}
    \mathcal{C}_{i,j}^{(n)} & =\sum\nolimits _{k=1}^{n-1}\!\!\!\!\sum\limits _{\overset{m_1 +\cdots +m_k =n-1}{m_l  >0}}\!\!\binom{n-1}{m_1 m_2 \cdots m_k} \tilde{\Lambda}_k( X_i \!\leadsto\! X_j).
    \end{aligned}
    \end{equation}
    where $\binom{n-1}{m_1 m_2 \cdots m_k} =\frac{(n-1)!}{m_1 !m_2 !\cdots m_k !}$ is the multinomial coefficients.
\end{theorem}

\begin{proof}
    Since $ i$ is not a root vertex, the structural equation of $ X_i$ is $ X_i =\sum _{m\in An( i)}\sum _{h=1}^{|P^{m\leadsto i} |} A_h^{m\leadsto i} \circ \epsilon _m +\epsilon _i$, where $ A_h^{m\leadsto i}$ is the sequence of coefficients corresponding to the $ h$-th path from $ m$, one of the ancestor of $ i$ ,to $ i$.

According the structural equation of $ X_i$, we have: 
\begin{equation}\begin{aligned}
\mathcal{C}_{i,j}^{(n)} & =\kappa \left( X_i ,\underbrace{X_j ,\dotsc ,X_j}_{n-1\text{times}}\right) =\kappa \left(\sum _{m\in An( i)}\sum _{h=1}^{|P^{m\leadsto i} |} A_h^{m\leadsto i} \circ \epsilon _m +\epsilon _i ,\underbrace{X_j ,\dotsc ,X_j}_{n-1\text{times}}\right)
\end{aligned}\end{equation}
then we decompose $ \mathcal{C}_{i,j}^{(n)}$ according to the structural equation of $ X_i$, we have:
\begin{equation}
\label{eq: Cij^n in T5}
\begin{aligned}
\mathcal{C}_{i,j}^{(n)} & =\sum _{m\in An( i)}\sum _{h=1}^{|P^{m\leadsto i} |} \kappa \left( A_h^{m\leadsto i} \circ \epsilon _m ,\underbrace{X_j ,\dotsc ,X_j}_{n-1\text{times}}\right) +\kappa \left( \epsilon _i ,\underbrace{X_j ,\dotsc ,X_j}_{n-1\text{times}}\right)\\
 & =\sum _{m\in An( i,j)}\sum _{h=1}^{|P^{m\leadsto i} |} \kappa \left( A_h^{m\leadsto i} \circ \epsilon _m ,\underbrace{X_j ,\dotsc ,X_j}_{n-1\text{times}}\right) +\kappa \left( \epsilon _i ,\underbrace{X_j ,\dotsc ,X_j}_{n-1\text{times}}\right)
\end{aligned}\end{equation}
As those cumulants involve independent noise components equal to zero, $ m$ is the common ancestor of $ i$ and $ j$ in the Eq. \ref{eq: Cij^n in T5}.

Now, we consider the decomposition of cumulant: (i) $ \kappa \left( A_h^{m\leadsto i} \circ \epsilon _m ,X_j ,\dotsc ,X_j\right)$ and (ii) $ \kappa ( \epsilon _i ,X_j ,\dotsc ,X_j)$. For (i), the term $ \kappa \left( A_h^{m\leadsto i} \circ \epsilon _m ,X_j ,\dotsc ,X_j\right)$ has the similar form as the cumulant we proved in Theorem 2, i.e.,
\begin{equation}\begin{aligned}
\kappa \left( A_h^{m\leadsto i} \circ \epsilon _m ,X_j ,\dotsc ,X_j\right) & =\sum _{l_1 =1}^{|\mathbf{P}^{m\leadsto j} |} \cdots \sum _{l_{n-1} =1}^{|\mathbf{P}^{m\leadsto j} |} \kappa \left( A_h^{m\leadsto i} \circ \epsilon _m ,A_{l_1}^{m\leadsto j} \circ \epsilon _m ,A_{l_2}^{m\leadsto j} \circ \epsilon _m ,\dotsc ,A_{l_{n-1}}^{m\leadsto j} \circ \epsilon _m\right)
\end{aligned}\end{equation}
where the only difference is the first noise component, which is $ A_h^{m\leadsto i} \circ \epsilon _m$ instead of $ \epsilon _m$. This variation does not impact the result in Theorem 2 leading to
\begin{equation}
\label{eq: copy T2 in T4 1}
\begin{aligned}
\kappa \left( A_h^{m\leadsto i} \circ \epsilon _m ,X_j ,\dotsc ,X_j\right) & =\sum _{k=1}^{n-1}\sum\limits _{\overset{m_1 +\cdots +m_k =n-1}{m_l  >0}}\binom{n-1}{m_1 m_2 \cdots m_k} \Lambda _k^{m\leadsto j} (A_h^{m\leadsto i} \circ \epsilon _m \leadsto X_j )
\end{aligned}\end{equation}
For (ii), $ \kappa ( \epsilon _i ,X_j ,\dotsc ,X_j)$ has the same form as the cumulant we proved in Theorem 2, we have:
\begin{equation}
\label{eq: copy T2 in T4 2}
\begin{aligned}
\kappa ( \epsilon _i ,X_j ,\dotsc ,X_j) & =\sum\nolimits _{k=1}^{n-1}\sum\limits _{\overset{m_1 +\cdots +m_k =n-1}{m_l  >0}}\binom{n-1}{m_1 m_2 \cdots m_k} \Lambda _k^{i\leadsto j} (1\circ X_i \leadsto X_j ).
\end{aligned}\end{equation}
Substituting Eq. \ref{eq: copy T2 in T4 1} and Eq. \ref{eq: copy T2 in T4 2} into Eq. \ref{eq: Cij^n in T5}, we have
\begin{equation}\label{eq:cumulant_mid}
\begin{aligned}
\mathcal{C}_{i,j}^{(n)} & =\sum _{m\in An( i,j)}\sum _{h=1}^{|P^{m\leadsto i} |}\sum _{k=1}^{n-1}\sum\limits _{\overset{m_1 +\cdots +m_k =n-1}{m_l  >0}}\binom{n-1}{m_1 m_2 \cdots m_k} \Lambda _k^{m\leadsto j} (A_h^{m\leadsto i} \circ \epsilon _m \leadsto X_j )\\
 & +\sum\nolimits _{k=1}^{n-1}\sum\limits _{\overset{m_1 +\cdots +m_k =n-1}{m_l  >0}}\binom{n-1}{m_1 m_2 \cdots m_k} \Lambda _k^{i\leadsto j} (1\circ X_i \leadsto X_j ).
\end{aligned}\end{equation}
By rewriting Eq. \ref{eq:cumulant_mid}, we have
\begin{equation}\begin{aligned}
\mathcal{C}_{i,j}^{(n)} & =\sum _{k=1}^{n-1}\sum\limits _{\overset{m_1 +\cdots +m_k =n-1}{m_l  >0}}\binom{n-1}{m_1 m_2 \cdots m_k}\left[\sum _{m\in An( i,j)}\sum _{h=1}^{|P^{m\leadsto i} |} \Lambda _k^{m\leadsto j} (A_h^{m\leadsto i} \circ \epsilon _m \leadsto X_j )+\Lambda _k^{i\leadsto j} (1\circ X_i \leadsto X_j )\right]\\
 & =\sum _{k=1}^{n-1}\sum\limits _{\overset{m_1 +\cdots +m_k =n-1}{m_l  >0}}\binom{n-1}{m_1 m_2 \cdots m_k} \tilde{\Lambda }_k (X_i \leadsto X_j ).
\end{aligned}\end{equation}
This completes the proof.
\end{proof}

\section{Proof of Theorem 6 and Theorem 7}
\begin{theorem}[Identification of PB-SCM]\label{thm:identification}
    If there exist $k\in \mathbb{Z}^+$ such that $\tilde{\Lambda}_k(X_i\leadsto X_j)\ne0$ and $\tilde{\Lambda}_k(X_j\leadsto X_i)=0$ for any two adjacency vertex $i$ and $j$, then $X_i$ is the ancestor of $X_j$
\end{theorem}

\begin{proof}
    For the case that $X_i$ is a root vertex. Suppose $X_i$ is not the ancestor of $X_j$, then $X_i$ and $X_j$ are independent since $X_i$ is the root vertex and $X_j$ is not the ancestor of $X_i$. In this case, the $\tilde{\Lambda}_k(X_i\leadsto X_j)=0$ for each $k$ since $X_i$ and $X_j$ are independent.
    
    For the case that $X_i$ is not a root vertex. Suppose $X_i$ is not the ancestor of $X_j$, then there must exist $k$ path from the common ancestor to $X_i$ since $\tilde{\Lambda}_k(X_i\leadsto X_j)\ne0$. 
    However, this contradicts the condition $\tilde{\Lambda }_k (X_j \leadsto X_i )=0$ as it indicates that there not exist $k$ paths from the common ancestor to $X_i$. Hence, we conclude that $X_i$ is the ancestor of $X_j$.
\end{proof}

\begin{theorem}[Graphical Implication of Identifiability]\label{th:graphical criteria 2}
    For a pair of vertices $i$ and $j$, if $i$ is an ancestor of $j$. The causal order of $i,j$ is identifiable by Theorem \ref{thm:identification}, if (i) vertex $i$ is a root vertex and $|\mathbf{P}^{i\leadsto j}|\geq 2$; or (ii) there exists a common ancestor $ {k\in \underset{l}{\arg\max}\left\{|\mathbf{P}^{l\leadsto i} ||l\in An( i,j)\right\}}$ has a directed path from $k$ to $j$ without passing $i$ in $G$. 
\end{theorem}

\begin{proof}
    (i) If vertex $i$ is a root vertex and $|\mathbf{P}^{i\leadsto j} |\geq 2$, we have
\begin{equation}\begin{aligned}
\tilde{\Lambda }_2 (X_i \leadsto X_j ) & =\Lambda _2^{i\leadsto j} (1\circ \epsilon _i \leadsto X_j )\neq 0\\
\tilde{\Lambda }_2 (X_j \leadsto X_i ) & =\sum _{h=1}^{|\mathbf{P}^{j\leadsto i} |} \Lambda _2^{j\leadsto i}\left( A_h^{j\leadsto i} \circ \epsilon _i \leadsto X_i\right) =0
\end{aligned}\end{equation}
Since there are $ \geqslant 2$ paths from $ i$ to $ j$ and no two paths from $ i$ to $ i$ (from noise component of $ i$ to $ i$), we have $ \tilde{\Lambda }_2 (X_i \leadsto X_j )\neq 0$ and $ \tilde{\Lambda }_2 (X_j \leadsto X_i )=0$. Based on Theorem 6, $i$ is the ancestor of $j$.

(ii) According to the acyclic constraints, the number of paths from their common ancestors to $ j$ is either equal or more than the number of paths from their common ancestors to $ i$, since $i$ is an ancestor of $j$ and those paths to $ i$ must reach $ j$.

If there exists a common ancestor $k\in \underset{l}{\arg\max}\left\{|\mathbf{P}^{l\leadsto i} ||l\in An(i,j)\right\}$ has a directed path from $k$ to $j$ without passing $i$ in $G$, it implies that the number of paths from $ k$ to $ j$ greater than that to $ i$,  Consequently, there must exist a value $ n=|\mathbf{P}^{k\leadsto j} |$ such that $ \tilde{\Lambda }_n (X_i \leadsto X_j )\neq 0$ and $ \tilde{\Lambda }_n (X_j \leadsto X_i )=0$ , and the causal order of $i,j$ is identifiable based on Theorem 6.
\end{proof}

\section{Proof of Theorem 8}

\begin{theorem}
    Let $G_{X_i |X_{Pa( i)}}\!( s)$ be the PGF of random variable $X_i$ given its parents variable $X_{Pa( i)}$, we have:
    \begin{equation}
    \label{eq: pgf}
    \begin{aligned}
     & P( X_i =k|X_{Pa( i)} =x_{Pa( i)})=\frac{1}{k!}\frac{\partial ^k G_{X_i |X_{Pa( i)}}( s)}{( \partial s)^k}\Bigl|_{s=0}\\
     & =\!\!\sum _{
        t_i +\!\!\sum\limits_{j\in Pa( i)}\!\! t_j =k
        }\!\! \frac{\mu _i^{t_i}\exp( -\mu _i)}{t_i !}\!\!\prod\limits _{j\in Pa(i)}\! \!\!\frac{( x_j)_{t_j}\alpha _{j,i}^{t_j} (1-\alpha _{j,i} )^{x_j -t_j}}{t_j !},
    \end{aligned}
    \end{equation}
where $t_j \leq x_j$, $ ( x_j)_{t_j} \!=\!\frac{x_j !}{( x_j -t_j) !}$ is the falling factorial, $\mu _i \!=\!E[ \epsilon _i]$, and $\epsilon _i$ is the noise component of $X_i$.
\end{theorem}

\begin{proof}
    For probability mass function $ P(X_i |X_{Pa(i)} )$, which can be decomposed as follow:
\begin{equation}\begin{aligned}
P(X_i |X_{Pa(i)} ) & =P( \epsilon _i)\prod _{j\in Pa(i)} P( \alpha _{j,i} \circ X_j |X_j) .
\end{aligned}\end{equation}
Let $G_{\epsilon _i} (s)$ represent the probability generating function (PGF) of $ P( \epsilon _i)$, which is the noise component of $ X_i$, and $G_{\alpha _{j,i} \circ X_j |X_j} (s)$ denote the PGF of $ P( \alpha _{j,i} \circ X_j |X_j)$, we have:
\begin{equation}\begin{aligned}
G_{X_i |X_{Pa(i)}} (s)=G_{\epsilon _i} (s)\prod _{j\in Pa(i)} G_{\alpha _{j,i} \circ X_j |X_j} (s),
\end{aligned}\end{equation}
where $G_{\epsilon _i} (s)=\exp [\mu _i (s-1)]$ and $G_{\alpha _{j,i} \circ X_j |X_j} (s)=(1-\alpha _{j,i} +\alpha _{j,i} s)^{X_j}$.

According to the property of PGF, we can calculate the probability mass function by taking derivatives of $ G_{X_i |X_{Pa(i)}} (s)$, and the derivative is expressed as:
\begin{equation}\begin{aligned}
\frac{\partial ^k G_{X_i |X_{Pa(i)}} (s)}{(\partial s)^k} & =\frac{\partial ^k\left( G_{\epsilon _i} (s)\prod _{j\in Pa(i)} G_{\alpha _{j,i} \circ X_j |X_j} (s)\right)}{(\partial s)^k}
\end{aligned}\end{equation}
According to the product rule of higher derivatives, i.e. 
\begin{equation}\begin{aligned}
\left(\prod _{i=1}^n f_i\right)^{(k)} & =\sum _{t_1 +t_2 +\cdots +t_n =k}\binom{k}{t_1 ,t_2 ,\dotsc ,t_n}\prod _{i=1}^n f_i^{(t_i )}\\
 & =\sum _{t_1 +t_2 +\cdots +t_n =k}\frac{k!}{t_1 !t_2 !\cdots t_n !}\prod _{i=1}^n f_i^{(t_i )} =k!\sum _{t_1 +t_2 +\cdots +t_n =k}\prod _{i=1}^n\frac{f_i^{(t_i )}}{t_i !} ,
\end{aligned}\end{equation}
we have:
\begin{equation}\begin{aligned}
\frac{\partial ^k G_{X_i |X_{Pa(i)}} (s)}{(\partial s)^k} & =k!\sum _{t_i +\sum\limits _{j\in Pa(i)} t_j =k}\frac{G_{\epsilon _i}^{( t_i)} (s)}{t_i !}\prod _{j\in Pa(i)}\frac{G_{\alpha _{j,i} \circ X_j |X_j}^{( t_j)} (s)}{t_j !} .
\end{aligned}\end{equation}
Furthermore, we have
\begin{equation}\begin{aligned}
G_{\epsilon _i}^{( t_i)} (s) & =\mu _i^{t_i}\exp (\mu _i( s-1) ),\\
G_{\alpha _{j,i} \circ X_j |X_j}^{( t_j)} (s) & =\begin{cases}
(X_j )_{t_j} \alpha _{j,i}^{t_j} (1-\alpha _{j,i} )^{X_j -t_j} & t_j \leqslant X_j\\
0 & t_j  >X_j
\end{cases}.
\end{aligned}\end{equation}
Given $ X_i =x_i$, $ X_{Pa(i)} =x_{Pa(i)}$, along with the model parameter $\Theta =\left\{\mathbf{A} =[\alpha _{i,j} ]\in [0,1]^{|V|\times |V|} ,\boldsymbol{\mu } =[\mu _i ]\in \mathbb{R}_{\geq 0}^{|V|}\right\}$, we can compute the probability mass function $ P(X_i |X_{Pa(i)} )$ as follow:

\begin{equation}
\label{eq: solution of likelihood}
\begin{aligned}
P(X_i =k|X_{Pa(i)} =x_{Pa(i)} ) & =\frac{1}{k!}\frac{\partial ^k G_{X_i |X_{Pa(i)} =x_{Pa(i)}} (s)}{(\partial s)^k}\Bigl|_{s=0}\\
 & =\frac{1}{k!}\sum _{t_i +\sum\limits _{j\in Pa(i)} t_j =k}\frac{G_{\epsilon _i}^{( t_i)} (0)}{t_i !}\prod _{j\in Pa(i)}\frac{G_{\alpha _{j,i} \circ X_j |X_j =x_j}^{( t_j)} (0)}{t_j !}\\
 & =\frac{1}{k!}\sum _{t_i +\sum\limits _{j\in Pa(i)} t_j =k}\frac{G_{\epsilon _i}^{( t_i)} (0)}{t_i !}\prod _{j\in Pa(i)}\frac{G_{\alpha _{j,i} \circ X_j |X_j =x_j}^{( t_j)} (0)}{t_j !}\\
 & =\sum _{t_i +\sum\limits _{j\in Pa(i)} t_j =k,}\frac{\mu _i^{t_i}\exp (-\mu _i )}{t_i !}\prod\limits _{j\in Pa(i)}\frac{(x_j )_{t_j} \alpha _{j,i}^{t_j} (1-\alpha _{j,i} )^{x_j -t_j}}{t_j !},
\end{aligned}\end{equation}
where $ t_j \leqslant x_j$. This completes the proof.
\end{proof}

\subsection{Accelerating Likelihood Computation Using FFT}
To compute the likelihood function, we have to calculate the Eq. \ref{eq: solution of likelihood}. However, this task remains computationally intensive due to the numerous parameter combinations satisfying the specific summation condition $ t_i +\sum\limits _{j\in Pa(i)} t_j =k$ and $ t_j \leqslant x_j$.

To address this issue, we show that the likelihood in Eq. \ref{eq: pgf} can be formulated as the problem of obtaining the coefficient of a polynomial product.

Specifically, the production of polynomials can be constructed as follows:
\begin{equation}
\label{eq: production of polynomial}
\begin{aligned}
    F( y) & =\left(\frac{\mu _i^0\exp (-\mu _i )}{0!} +\frac{\mu _i^1\exp (-\mu _i )}{1!} y+\cdots +\frac{\mu _i^k\exp (-\mu _i )}{k!} y^k\right)\\
     & \ \ \ \ \ \times \prod\limits _{j\in Pa(i)}\left(\frac{(x_j )_0 \alpha _{j,i} (1-\alpha _{j,i} )^{x_j}}{0!} +\frac{(x_j )_1 \alpha _{j,i} (1-\alpha _{j,i} )^{x_j -1}}{1!} y+\cdots +\frac{(x_j )_{x_j} \alpha _{j,i} (1-\alpha _{j,i} )^{x_j -x_j}}{x_j !} y^{x_j}\right)
\end{aligned}
\end{equation}

Then the likelihood in Eq. \ref{eq: solution of likelihood} is exactly the coefficient of $y^k$ after the production. To obtain the coefficient of such production, we can employ the Fast Fourier Transform (FFT). In detail, we can create a series of vectors of the coefficient of each polynomial in \ref{eq: production of polynomial}, and pad the list with 0 since the highest power of $ x$ is $ k\times |Pa( i) |$:
\begin{equation}
\label{}
    \begin{aligned}
\mathbf{a}_0 & =\left[\frac{\mu _i^0\exp (-\mu _i )}{0!} ,\frac{\mu _i^1\exp (-\mu _i )}{1!} ,...,\frac{\mu _i^k\exp (-\mu _i )}{k!} ,\underbrace{0,...,0}_{k\times |Pa( i) |-k+1\text{ times}}\right]\\
\mathbf{a}_{j_p} & =\left[\frac{(x_{j_p} )_0 \alpha _{j_p ,i} (1-\alpha _{j_p ,i} )^{x_{j_p}}}{0!} ,\frac{(x_{j_p} )_1 \alpha _{j_p ,i} (1-\alpha _{j_p ,i} )^{x_{j_p} -1}}{1!} ,...,\frac{(x_{j_p} )_{x_{j_p}} \alpha _{j_p ,i} (1-\alpha _{j_p ,i} )^{x_{j_p} -x_{j_p}}}{x_{j_p} !} ,\underbrace{0,...,0}_{k\times |Pa( i) |-x_{j_p} +1\text{ times}}\right]
\end{aligned}
\end{equation}
where $j_p\in Pa(i)$ and $p=1,2,...,|Pa(i)|$. Then the coefficient vector of the expansion of Eq. \ref{eq: production of polynomial} is given by:
\begin{equation}\begin{aligned}
\hat{\mathbf{a}} & =\text{IFFT}\left(\text{FFT}(\mathbf{a}_0) \odot \text{FFT}(\mathbf{a}_{j_1}) \odot \cdots \odot \text{FFT}(\mathbf{a}_{j_{|Pa( i) |}})\right)
\end{aligned}\end{equation}
Here, $ \odot $ is the element-wise multiplication, $ \text{FFT}( \cdot )$ and $ \text{IFFT}( \cdot )$ is the implementation of fast Fourier transform and Inverse fast Fourier transform respectively. Consequently, the $ k+1$-th element in the vector $ \hat{\mathbf{a}}$ is the coefficient of $ y^k$ in the expansion of Eq. \ref{eq: production of polynomial}, which is the likelihood given Eq. \ref{eq: solution of likelihood}.

\section{Additional Experiments}
The main paper has shown the F1 scores and other baselines in synthetic data experiments. Here, we further provide the Precision, Recall, and Structural Hamming Distance (SHD) in these experiments, as shown in Fig. \ref{fig:sensitivity Precision}, Fig. \ref{fig:sensitivity Recall} and Fig. \ref{fig:sensitivity SHD}. 

\begin{figure*}[h]
	\centering
	\subfigure[Sensitivity to Avg. Indegree Rate]{
	\includegraphics[width=0.31\textwidth]{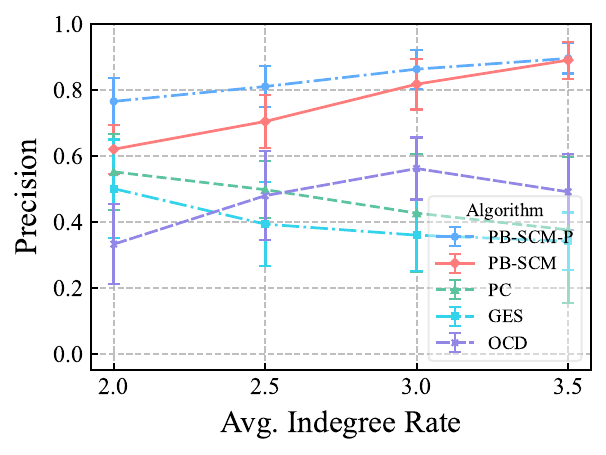}
	\label{fig:sensitivity: Indegree Precision}
}
	\subfigure[Sensitivity to Number of vertices]{
		\includegraphics[width=0.31\textwidth]{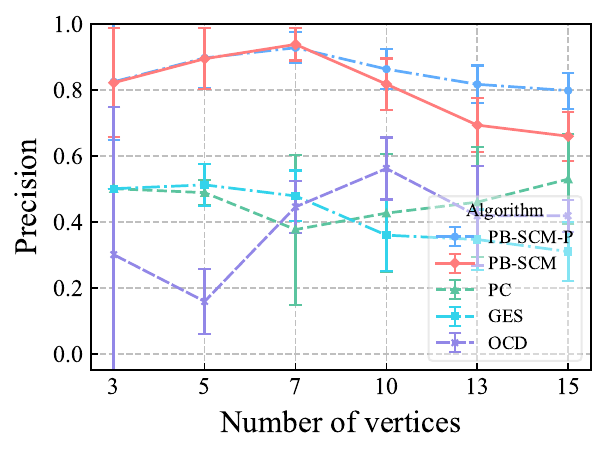}
		\label{fig:sensitivity: Node Precision}
	}
	\subfigure[Sensitivity to Sample Size]{
	\includegraphics[width=0.32\textwidth]{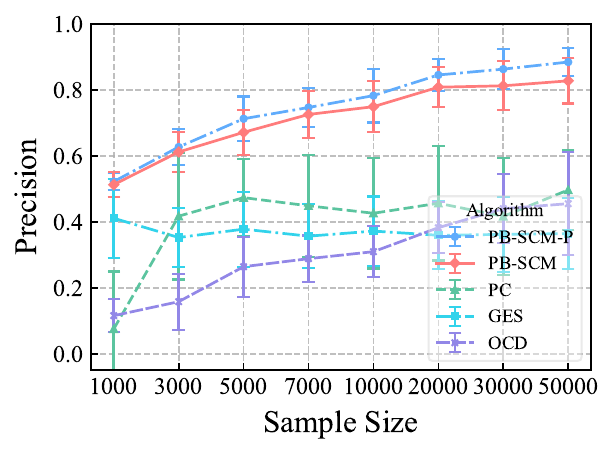}
	\label{fig:sensitivity: sample Precision}
}
	\caption{Precision in the Sensitivity Experiments}	
	\label{fig:sensitivity Precision}
\end{figure*}

\begin{figure*}[h]
	\centering
	\subfigure[Sensitivity to Avg. Indegree Rate]{
	\includegraphics[width=0.31\textwidth]{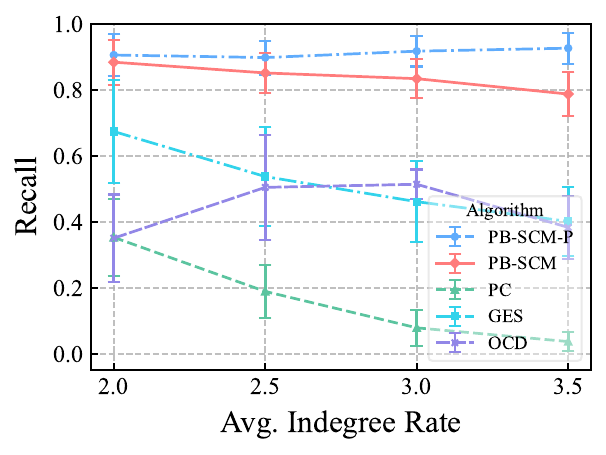}
	\label{fig:sensitivity: Indegree Recall}
}
	\subfigure[Sensitivity to Number of vertices]{
		\includegraphics[width=0.31\textwidth]{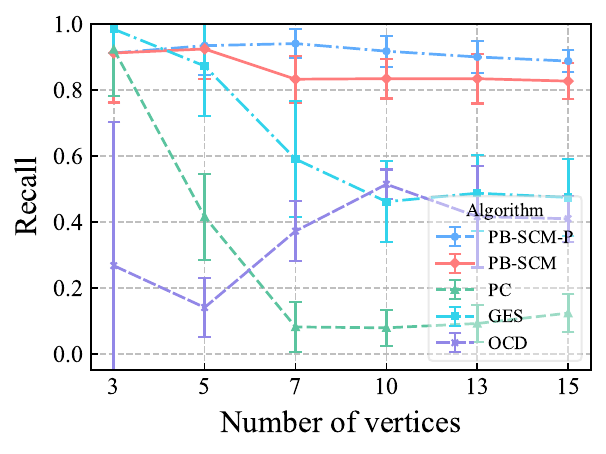}
		\label{fig:sensitivity: Node Recall}
	}
	\subfigure[Sensitivity to Sample Size]{
	\includegraphics[width=0.32\textwidth]{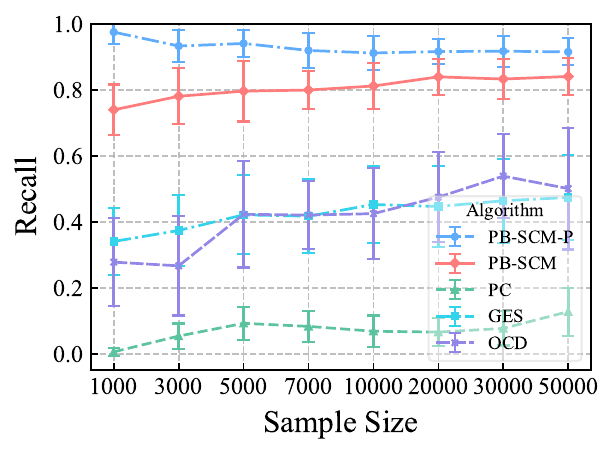}
	\label{fig:sensitivity: sample Recall}
}
	\caption{Recall in the Sensitivity Experiments}	
	\label{fig:sensitivity Recall}
\end{figure*}

\begin{figure*}[h]
	\centering
	\subfigure[Sensitivity to Avg. Indegree Rate]{
	\includegraphics[width=0.31\textwidth]{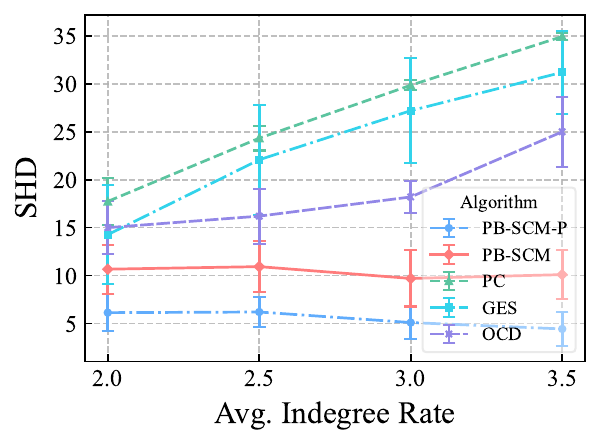}
	\label{fig:sensitivity: Indegree SHD}
}
	\subfigure[Sensitivity to Number of vertices]{
		\includegraphics[width=0.31\textwidth]{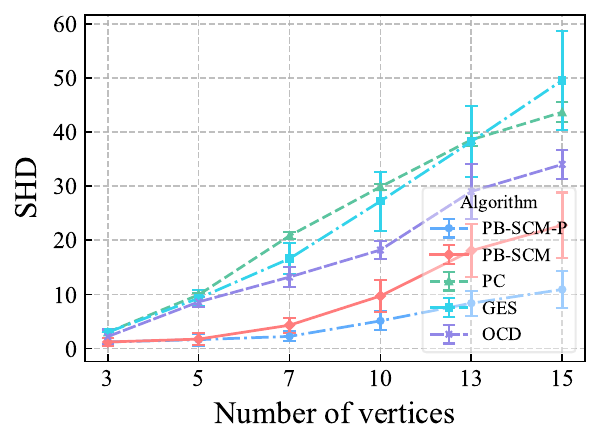}
		\label{fig:sensitivity: Node SHD}
	}
	\subfigure[Sensitivity to Sample Size]{
	\includegraphics[width=0.32\textwidth]{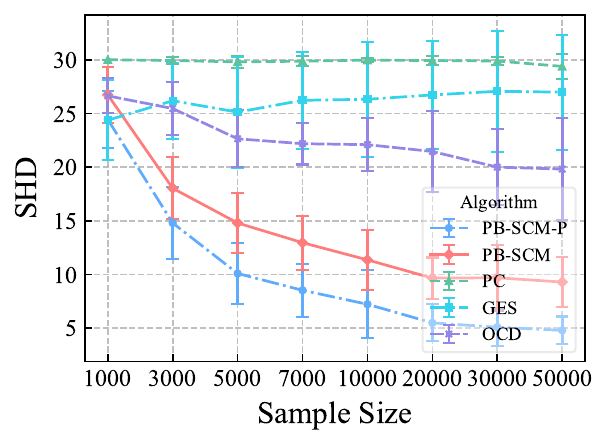}
	\label{fig:sensitivity: sample SHD}
}
	\caption{SHD in the Sensitivity Experiments}	
	\label{fig:sensitivity SHD}
\end{figure*}

\begin{table}[h]
  \centering
  \caption{Sensitivity to the max order of cumulant $K$}
  \label{tab:data}
  \begin{tabular}{ccccccc}
    \toprule
    \multirow{2}{*}{K} & \multicolumn{4}{c}{Score type} \\
    \cmidrule(lr){2-5}
    & F1 & Precision & Recall & SHD \\
    \midrule
    2 & $0.69 \pm 0.04$ & $0.56 \pm 0.05$ & $0.90 \pm 0.05$ & $23.38 \pm 3.88$ \\
    3 & $0.82 \pm 0.06$ & $0.82 \pm 0.08$ & $0.83 \pm 0.06$ & $9.53 \pm 2.71$ \\
    4 & $0.82 \pm 0.06$ & $0.82 \pm 0.07$ & $0.83 \pm 0.06$ & $9.47 \pm 2.78$ \\
    5 & $0.83 \pm 0.05$ & $0.82 \pm 0.07$ & $0.84 \pm 0.05$ & $9.47 \pm 2.68$ \\
    \bottomrule
  \end{tabular}
\end{table}

\end{document}